\documentclass[format=acmsmall, review=false, screen=true]{acmart}

\usepackage{booktabs} 

\usepackage{graphicx}
\usepackage{tikz}
\usetikzlibrary{fit,positioning,arrows,automata,calc}
\tikzset{
  main/.style={circle, minimum size = 5mm, thick, draw =black!80, node distance = 10mm},
  connect/.style={-latex, thick},
  box/.style={rectangle, draw=black!100}
}

\usepackage{amsmath}
\usepackage[colorinlistoftodos]{todonotes}
\usepackage{bm}
\usepackage{breqn}
\usepackage{physics}
\usepackage{enumitem}
\usepackage{bbold}
\usepackage{epigraph}

\usepackage{csquotes}
\usepackage{hyperref}       
\usepackage{url}            
\usepackage{booktabs}       
\usepackage{amsfonts}       
\usepackage{nicefrac}       
\usepackage{microtype}      
\usepackage{color}
\usepackage{scalerel}
\usepackage[toc,page]{appendix}
\usepackage{tabularx,ragged2e,booktabs,caption}
\usepackage[font=small,labelfont=bf]{caption}
\usepackage{blindtext}
\usepackage{amssymb}
\usepackage{mathtools}
\usepackage{dsfont}
\usepackage{amsmath}
\allowdisplaybreaks
\usepackage{amsfonts}
\usepackage{amsthm}
\usepackage{float}
\usepackage{graphicx}
\usepackage{subcaption}
\graphicspath{ {images/} }
\usepackage{letltxmacro}%
\usepackage{thmtools,thm-restate}
\usepackage{relsize}
\usepackage{algpseudocode,algorithm,algorithmicx}

\newtheorem{definition}{Definition}[section]
\newtheorem{thm}{Theorem}[section]
\newtheorem*{thm*}{Theorem}
\newtheorem{cor}{Corollary}[section]
\newtheorem*{cor*}{Corollary}
\newtheorem{lemma}{Lemma}[section]
\newtheorem{prop}{Proposition}[section]
\newtheorem*{prop*}{Proposition}
\newtheorem{property}{Property}[section]

\DeclareMathOperator*{\argmin}{arg\,min}

\newcommand{\bX}{\boldsymbol{X}}
\newcommand{\bY}{\boldsymbol{Y}}
\newcommand{\bM}{\boldsymbol{M}}

\newcommand{\bQ}{\boldsymbol{Q}}
\newcommand{\bA}{\boldsymbol{A}}

\newcommand{\bI}{\boldsymbol{I}}
\newcommand{\bU}{\boldsymbol{U}}
\newcommand{\bV}{\boldsymbol{V}}

\newcommand{\bS}{\boldsymbol{S}}

\newcommand{\btA}{\widetilde{\bA}}
\newcommand{\btM}{\widetilde{\bM}}

\newcommand{\btX}{\widetilde{\bX}}

\newcommand{\bhM}{\widehat{\bM}}
\newcommand{\bhQ}{\widehat{\bQ}}

\newcommand{\bhtM}{\widehat{\widetilde{\bM}}}
\newcommand{\bhtkM}{{\widehat{\widetilde{\bM}}}^{(k)}}

\newcommand{\lineM}{\overline{M}}
\newcommand{\blineM}{\overline{\bM}}

\newcommand{\Ex}{\mathbb{E}}
\newcommand{\Pb}{\mathbb{P}}

\newcommand{\lcm}{\text{lcm}}

\AtBeginDocument{\renewcommand{\Omega}{\varOmega}}

\AtBeginDocument{\renewcommand{\Delta}{\varDelta}}

\setcopyright{acmcopyright}
\acmJournal{POMACS}
\acmYear{2018}
\acmVolume{2}
\acmNumber{3}
\acmArticle{40}
\acmMonth{12}
\acmPrice{15.00}
\acmDOI{10.1145/3287319}

\setcopyright{acmlicensed}

\received{August 2018}
\received[revised]{October 2018}
\received[accepted]{December 2018}

\begin{document}
\title[Model Agnostic Time Series Analysis via Matrix Estimation]{Model Agnostic Time Series Analysis via Matrix Estimation}

\author{Anish Agarwal}
\affiliation{%
  \institution{Massachusetts Institute of Technology}
  \streetaddress{32-D666 Vassar St.}
  \city{Cambridge}
  \state{MA}
  \postcode{02139}
  \country{USA}}
\email{anish90@mit.edu}

\author{Muhammad Jehangir Amjad}
\affiliation{%
  \institution{Massachusetts Institute of Technology}
  \streetaddress{32-D560 Vassar St.}
  \city{Cambridge}
  \state{MA}
  \postcode{02139}
  \country{USA}}
\email{mamjad@mit.edu}

\author{Devavrat Shah}
\affiliation{%
  \institution{Massachusetts Institute of Technology}
  \streetaddress{32-D670 Vassar St.}
  \city{Cambridge}
  \state{MA}
  \postcode{02139}
  \country{USA}}
\email{devavrat@mit.edu}

\author{Dennis Shen}
\affiliation{%
  \institution{Massachusetts Institute of Technology}
  \streetaddress{32-D560 Vassar St.}
  \city{Cambridge}
  \state{MA}
  \postcode{02139}
  \country{USA}}
\email{deshen@mit.edu}

\begin{abstract}
We propose an algorithm to impute and forecast a time series by transforming the observed time series into a matrix, utilizing matrix estimation to recover missing values and de-noise observed entries, and performing linear regression to make predictions. 
At the core of our analysis is a representation result, which states that for a
large class of models, the transformed time series matrix 
is (approximately) low-rank. 
In effect, this generalizes the widely used Singular Spectrum Analysis (SSA) in the time series literature, 
and allows us to establish a rigorous link between time series analysis and matrix estimation. 
The key to establishing this link is constructing a Page matrix with non-overlapping entries rather 
than a Hankel matrix as is commonly done in the literature (e.g., SSA). 
This particular matrix structure allows us to provide finite sample analysis for imputation and prediction, and prove the asymptotic consistency 
of our method. 
Another salient feature of our algorithm is that it is model agnostic with respect to both the underlying
time dynamics and the noise distribution in the observations. 
The noise agnostic property of our approach allows us to recover the latent states when only given access to noisy and partial observations a la a Hidden Markov Model; e.g., recovering the time-varying parameter of a Poisson process {\em without knowing} that the underlying process is Poisson.
Furthermore, since our forecasting algorithm requires regression with noisy features, our approach suggests a matrix estimation based method---coupled with a novel, non-standard matrix estimation error metric---to solve the error-in-variable regression problem, which could be of interest in its own right. 
Through synthetic and real-world datasets, we demonstrate that our algorithm outperforms standard software packages (including R libraries) in the presence of missing data as well as high levels of noise.  

\end{abstract}

%
%
%

%
%


\maketitle

\renewcommand{\shortauthors}{A. Agarwal et al.}



\section{Introduction}\label{sec: Introduction}

Time series data is of enormous interest across all domains of life: from health sciences and weather forecasts to retail and finance, time dependent data is ubiquitous. Despite the diversity of applications, time series problems are commonly confronted by the same two pervasive obstacles: interpolation and extrapolation in the presence of noisy and/or missing data. 
Specifically, we consider a discrete-time setting with $t \in \mathbb{Z}$ representing the time index and $f: \mathbb{Z} \to \mathbb{R}$\footnote{We denote $\mathbb{R}$ as the field of real numbers and $\mathbb{Z}$ as the integers.} representing the latent discrete-time time series of interest. For each $t \in [T] := \{1, \dots, T\}$ and with probability $p \in (0, 1]$, we observe the random variable $X(t)$ such that $\Ex[X(t)] = f(t)$. While the underlying mean signal $f$ is of course strongly correlated, we assume the per-step noise is independent across $t$ and has uniformly bounded variance. Under this setting, we have two objectives: (1) interpolation, i.e., estimate $f(t)$ for all $t \in [T]$; (2) extrapolation, i.e., forecast $f(t)$ for $t > T$. Our interest is in designing a generic method for interpolation and extrapolation that is applicable to a large model class while being agnostic to the time dynamics and noise distribution.

We develop an algorithm based on matrix estimation, a topic which has received widespread attention, especially with the advent of large datasets. In the matrix estimation setting, there is a ``parameter'' matrix $\bM$ of interest, and we observe a sparse, corrupted signal matrix $\bX$ where $\Ex[\bX] = \bM$. The aim then is to recover the entries of $\bM$ from noisy and partial observations given in $\bX$. For our purposes, the attractiveness of matrix estimation derives from the property that these methods are fairly model agnostic in terms of the structure of $\bM$ and distribution of $\bX$ given $\bM$. We utilize this key property to develop a model and noise agnostic time series imputation and prediction algorithm. 


\subsection{Overview of contributions} \label{sec:contributions}

\smallskip
\noindent{\bf Time series as a matrix. } 
We transform the time series of observations $X(t)$ for $t \in [T]$ into what is known as the Page matrix (cf. \cite{pagematrix}) by placing contiguous segments of size $L > 1$ (an algorithmic hyper-parameter) of the time series into non-overlapping columns; see Figure \ref{fig:algo} for a caricature of this transformation. 

As the key contribution, we establish that---in expectation---this generated matrix is either exactly or {\em approximately} low-rank for a large class of models $f$.  Specifically, $f$ can be from the following families: 
\begin{enumerate}[wide, labelwidth=!, labelindent=0pt]
\item[]\textit{Linear Recurrent Formulae (LRF)}: $f(t) = \sum_{g = 1}^G \alpha_g f(t-g)$.

\item[]  \textit{Compact Support}: $f(t) = g(\varphi(t))$ where $\varphi: \mathbb{Z} \to [-C_1, C_1]$ has the form $\varphi(t + s) = \sum_{l=1}^{G} \alpha_l a_l(t) b_l(s)$ 
with $\alpha_l \in [-C_2, C_2], a_l: \mathbb{Z} \to [0,1], b_l: \mathbb{Z} \to [0,1]$ for some $C_1, C_2 > 0$; and $g: [-C_1, C_1] \to \mathbb{R}$ is $\mathcal{L}$-Lipschitz
\footnote{We say $g: \mathbb{R} \to \mathbb{R}$ is $\mathcal{L}$-Lipschitz if there exists a $\mathcal{L} \ge 0$ such that $\norm{g(x) - g(y)} \le \mathcal{L} \norm{x- y}$ for all $x, y \in \mathbb{R}$ and $\norm{\cdot}$ denotes the standard Euclidean norm on $\mathbb{R}$.}
\footnote{It can be verified that if $\varphi$ is an LRF satisfying $\varphi(t)  = \sum_{h = 1}^H \gamma_h \varphi(t-h)$, then it satisfies the form $\varphi(t + s) = \sum_{g=1}^{G} \alpha_g a_g(t) b_g(s)$ for $G = H$ with appropriately defined constants $\alpha_g$, functions 
$a_g, b_g$; see Proposition \ref{prop:lrf_decomposition} of Appendix \ref{sec:appendix:models} for details.}.

\item[]\textit{Sublinear}: $f(t) = g(t)$ where $g: {\mathbb R} \to {\mathbb R}$ and $\abs{\frac{d g(s)}{ds}} \le C s^{-\alpha}$ for some $\alpha, C> 0$, and $\forall s \in \mathbb{R}$.
\end{enumerate} 
Over the past decade, the matrix estimation community has developed a plethora of methods to recover an exact or approximately low-rank matrix from its noisy, partial observations in a noise and model agnostic manner. 
Therefore, by applying such a matrix estimation method to this transformed matrix, we can recover the underlying mean matrix (and thus $f(t)$ for $t \in [T]$) accurately. In other words, we can interpolate and de-noise the original corrupted and incomplete time series without any knowledge of its time dynamics or noise distribution. Theorem \ref{thm:imputation} and Corollary \ref{corollary:imputation} provide finite-sample analyses for this method and establish the consistency property of our algorithm, as long as the underlying $f$ satisfies Property \ref{prop.one} and the matrix estimation method satisfies Property \ref{prop:2.1}. In Section \ref{sec:model}, we show that any additive mixture of the three function classes listed above satisfies Property \ref{prop.one}. Effectively, Theorem \ref{thm:imputation} establishes a {\em statistical reduction} between time series imputation and matrix estimation. Our key contribution with regards to imputation lies in establishing that a large class of time series models (see Section \ref{sec:model}) satisfies Property \ref{prop.one}.  

\begin{figure}[H]
	\centering
	\includegraphics[width=0.5\textwidth]{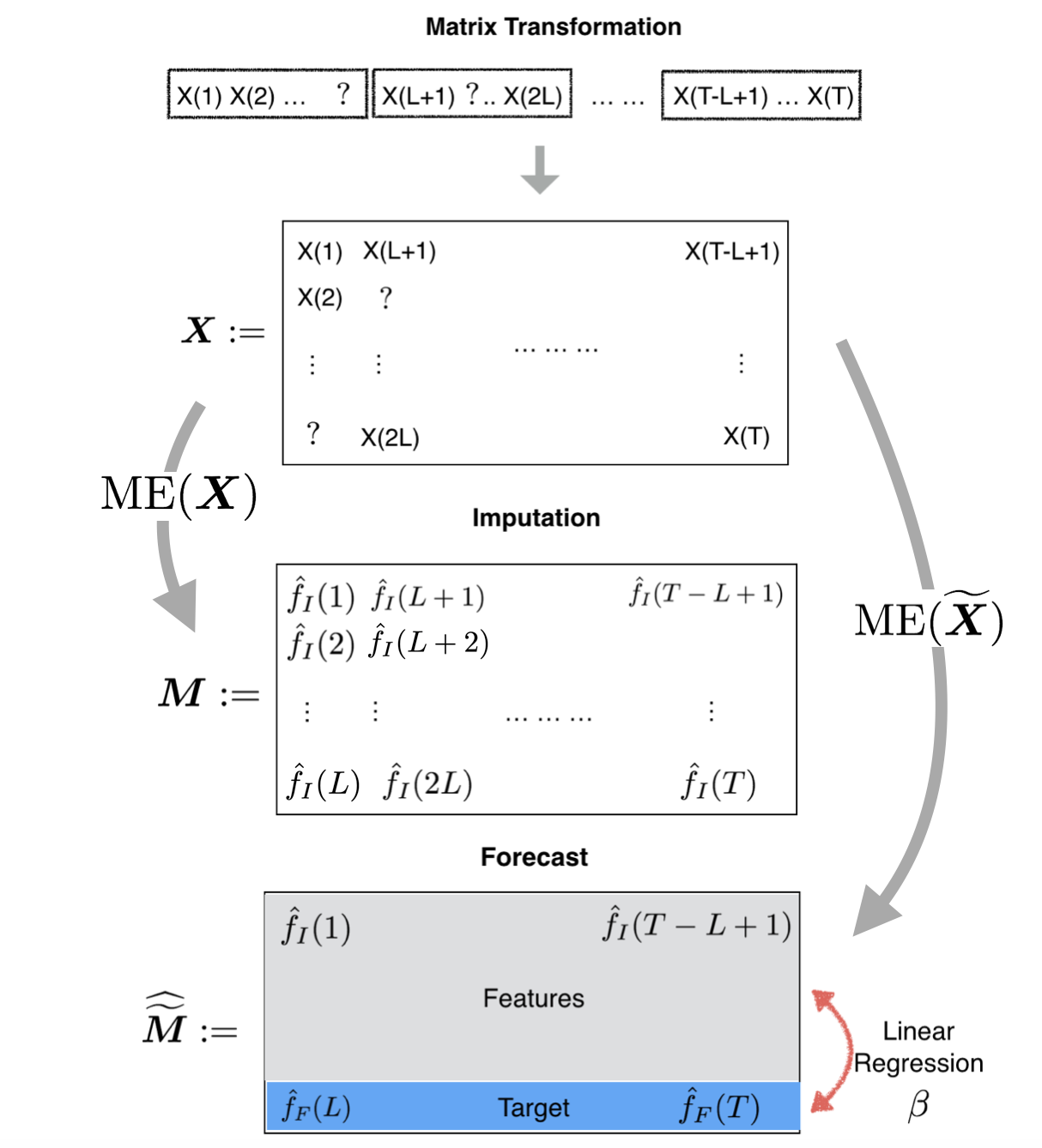}
	\caption{Caricature of imputation and forecast algorithms. We first transform the noisy time series $X(t)$ (with ``$?$'' indicating missing data) into a Page matrix $\bX$ with non-overlapping entries. For imputation, we apply a matrix estimation ($\text{ME}$) algorithm with input $\bX$ to obtain the estimates $\hat{f}_I(t)$ for the de-noised and filled-in entries. For forecasting, we first apply $\text{ME}$ to $\btX$ (i.e., $\bX$ excluding the last row), and then fit a linear model $\beta$ between the last row and all other rows to obtain the forecast estimates $\hat{f}_F(t)$.}
	\label{fig:algo}
\end{figure}

It is clear that for LRF, the last row of the mean transformed matrix can be expressed as a linear combination of the other rows. 
An important representation result of the present paper, which generalizes this notion, is that an {\em approximate} LRF relationship holds for the other two model classes. Therefore, we can forecast $f(t)$, say for $t = T + 1$, as follows: apply matrix estimation to the transformed data matrix as done in imputation; then, linearly regress the last row with respect to the other rows in the matrix; 
finally, compute the inner product of the learnt regression vector with the vector containing the previous $L-1$ values that were estimated via the matrix estimation method. Theorem \ref{thm:asymptotics} and Corollary \ref{cor:asymptotics} imply that the mean-squared error of our predictions decays to zero provided the matrix estimation 
method satisfies Property \ref{prop:2.2} and the underlying model $f$ satisfies Property \ref{prop.two}. Similar to the case of imputation, establishing that Property \ref{prop.two} holds for the three function classes is novel (see Section \ref{sec:model}).

%
%

\smallskip
\noindent {\bf Noisy regression.} Our proposed forecasting algorithm performs regression with noisy and incomplete features. In the literature, this is known as error-in-variable regression. Recently, there
has been exciting progress to understand this problem especially in the {\em high-dimensional} setting \cite{PoWainwright12, belloni2017linear, datta2017cocolasso}.
Our algorithm offers an alternate solution for the high-dimensional setting through the lens of matrix estimation: first, utilize matrix estimation to de-noise and impute the feature observations, and then perform least squares with the pre-processed feature matrix. We demonstrate that if the true, underlying feature matrix is (approximately) low-rank, then our algorithm provides a consistent estimator to the true signal (with finite sample guarantees). Our analysis further suggests the usage of a non-standard error metric, the max row sum error (MRSE) (see Property \ref{prop:2.2} for details). 


\smallskip
\noindent {\bf Class of applicable models.}  As aforementioned, our algorithm enjoys 
strong performance guarantees provided the underlying mean matrix induced by the time series $f$ satisfies 
certain structural properties, i.e., Properties \ref{prop.one} and \ref{prop.two}. We argue that a broad class of commonly used time series 
models meets the requirements of the three function classes listed above.

LRFs include the following important family of time series: a finite sum of products of exponentials ($\exp{\alpha t}$), harmonics ($\cos(2\pi \omega t + \phi) $), and finite degree polynomials ($P_m(t)$) \cite{golyandina2001analysis}, i.e., $f(t)= \sum_{g=1}^G \exp{\alpha_g t} \cos(2\pi \omega_g t + \phi_g) P_{m_g}(t)$. Further, since stationary processes and $L_2$ integrable functions are well approximated by a finite summation of harmonics (i.e., $\sine$ and $\cosine$), LRFs encompass a vitally important family of models. For this model, we show that indeed the structural properties required from the time series matrix for both imputation and prediction are satisfied. 

However, there are many important time series models that do not admit a finite order LRF representation. A few toy examples include $\cos(\sin(t)), ~\exp{\sin^2(t)}, ~\log{t}, ~\sqrt{t}$. 
Time series models with compact support, on the other hand, include models composed of a finite summation of periodic functions (e.g., $\cos(\sin(t)), ~\exp{\sin^2(t)}$). Utilizing our low-rank representation result, we establish that models with compact support possess the desired structural properties. 
%
We further demonstrate that sublinear functions, which include models that are composed of a finite summation of non (super-)linear functions (e.g., $\log t, ~\sqrt{t}$), also possess the necessary structural properties.
%
Importantly, we argue that the finite mixture of the above processes satisfy the necessary structural properties.

\smallskip
\noindent {\bf Recovering the hidden state.} Our algorithm, being noise and time-dynamics agnostic, makes it relevant to recover
the hidden state from its noisy, partial observations as in a Hidden Markov-like Model. For example, imagine having
access to partial observations of a time-varying truncated Poisson process\footnote{Let $C$ denote a positive, bounded constant, and $X$ a Poisson random variable. We define the {\it truncated} Poisson random variable $Y$ as $Y = \min\{X, C\}$.} {\em without} knowledge that the process is Poisson. By
applying our imputation algorithm, we can recover time-varying parameters of this process accurately and, thus, the hidden states. If we were to apply an Expectation-Maximization (EM) like algorithm, it would require knowledge
of the underlying model being Poisson; moreover, theoretical guarantees are not clear for such an approach.


\smallskip
\noindent {\bf Sample complexity.} Given the generality and model agnostic nature of our algorithm, it is expected that its sample complexity for a specific model class will be worse than model aware optimal algorithms. Interestingly, our finite sample analysis suggests that for the model classes stated above, the performance loss incurred due to this generality is minor. See Section \ref{ssec:sc} for a detailed analysis. 

\smallskip
\noindent {\bf Experiments.} Using synthetic and real-world datasets, our experiments establish that 
our method \textit{outperforms existing standard software packages} (including R) for the tasks of interpolation 
and extrapolation in the presence of noisy and missing observations. When the data is generated synthetically, 
we ``help" the existing software package by choosing the correct parametric model and algorithm while our 
algorithm remains oblivious to the underlying model; despite this disadvantage, our algorithm continues to 
\textit{outperform} the standard packages with missing data. 

Further, our empirical studies demonstrate that our imputation algorithm accurately recovers the hidden state for Hidden Markov-like Models, verifying our theoretical imputation guarantees (see Theorem \ref{thm:imputation}). All experimental findings can be found in Section \ref{sec:experiments_appendix}.

\subsection{Related works} \label{sec:related_works}

There are two related topics: matrix estimation and time series analysis. Given the richness of both fields, we cannot do justice in providing a full overview. Instead, we provide a high-level summary of known results with references that provide details.

\smallskip
\noindent {\bf Matrix estimation.}  Matrix estimation is the problem of recovering a data matrix from an incomplete and noisy sampling of its entries. This has become of great interest due to its connection to 
recommendation systems (cf. \cite{KeshavanMontanariOh10a, KeshavanMontanariOh10b, NegahbanWainwright11, ChenWainwright15, Chatterjee15, LeeLiShahSong16, CandesTao10, Recht11, DavenportPlanBergWootters14}), 
social network analysis (cf.  \cite{AbbeSandon15a, AbbeSandon15b,  AbbeSandon15c, AnandkumarGeHsuKakade13, hopkins2017efficient}), 
and graph learning (graphon estimation) (cf. \cite{AiroldiCostaChan13, ZhangLevinaZhu15, BorgsChayesCohnGanguly15, BorgsChayesLeeShah17}). 
The key realization of this rich literature is that one can estimate the true underlying matrix from noisy, partial observations by simply taking a low-rank approximation of the observed data. We refer an interested reader to recent works such as \cite{Chatterjee15, BorgsChayesLeeShah17} and references there in.  

\smallskip
\noindent {\bf Time series analysis.} The question of time series analysis is potentially as old as civilization in some form. Few textbook style references include \cite{timeseries1, arima1, timeseries2, timeseries-survey}. At the highest level, time series modeling primarily involves viewing a given time series as a function indexed by time (integer or real values) and the goal of model learning is to identify this function from observations (over finite intervals). Given that the space of such functions is complex, the task is to utilize function form (i.e., ``basis functions'') so that for the given setting, the time series observation can fit a sparse representation. 
For example, in communication and signal processing, the harmonic or Fourier representation of a time series has been widely utilized, due to the fact that signals communicated are periodic in nature. 
The approximation of stationary processes via harmonics or ARIMA has made them a popular model class to learn stationary-like time series, with domain specific popular variations, such as `Autoregressive Conditional Heteroskedasticity' (ARCH) in finance. To capture non-stationary or ``trend-like'' behavior, polynomial bases have been considered. There are rich connections to the theory of stochastic processes and information theory (cf. \cite{Cover1, shields1998interactions, Rissanen, Feder}). Popular time series models with latent structure are Hidden Markov Models (HMM) in probabilistic form (cf. \cite {Kalman, HMM} and Recurrent Neural Networks (RNN) in deterministic form (cf. \cite{RNN}). 

The question of learning time series models with missing data has received comparatively less attention. A common 
approach is to utilize HMMs or general State-Space-Models to learn with missing data (cf. \cite{missing0, missing1}). 
To the best of the authors' knowledge, most work within this literature is restricted to such class of models (cf. \cite{missing2}). Recently, building on the literature in online learning, sequential approaches have been proposed to address prediction with missing data (cf. \cite{anava2015}). 

\vspace{.05in}
\noindent {\bf Time series and matrix estimation.} The use of a matrix structure for time series analysis has roughly two 
streams of related work: SSA for a single time series (as in our setting), and the use of multiple time series. We discuss relevant results for both of these topics.

\vspace{.05in}
\noindent {\em Singular Spectrum Analysis (SSA)} of time series has been around for some time. {Generally, it assumes access to time series data that is not noisy and fully observed.} The core steps of SSA for a given time series are as follows: (1) create a Hankel matrix from the time series data; (2) perform a Singular Value Decomposition (SVD) of it; (3) group the singular values based on user belief of the model that generated the process; (4) perform diagonal averaging for the ``Hankelization" of the grouped rank-1 matrices outputted from the SVD to create a set of time series; (5) learn a linear model for each ``Hankelized" time series for the purpose of forecasting. 

At the highest level, SSA and our algorithm are cosmetically similar to one another. There are, however, several key differences: 
(i) {\it matrix transformation}---while SSA uses a Hankel matrix (with repeated entries), we transform the time series into a Page matrix (with non-overlapping structure); 
(ii) {\it matrix estimation}---SSA heavily relies on the SVD while we utilize general matrix estimation procedures (with SVD methods representing one specific procedural choice); 
(iii) {\it linear regression}---SSA assumes access to fully observed and noiseless data while we allow for corrupted and missing entries.

These differences are key in being able to derive theoretical results. For example, there have been numerous recent works that have attempted to apply matrix estimation methods to the Hankel matrix inspired by SSA for imputation, but these works do not provide any theoretical guarantees \cite{shen2015improved, schoellhamer2001singular, tsagkatakis2016singular}. In effect, the Hankel structure creates strong correlation of noise in the matrix, which is an impediment for proving theoretical results. Our use of the Page matrix overcomes this challenge and we argue that in doing so, we still retain the underlying structure in the matrix. With regards to forecasting, the use of matrix estimation methods that provide guarantees with respect to MRSE rather than standard MSE is needed (which SSA provides no theoretical analysis for). While we do not explicitly discuss such methods in this work, such methods are explored in detail in ~\cite{superviselearninghighdim}.
With regards to imputation, SSA does not provide direction on how to group the singular values, which is instead done based on user belief of the generating process. However, due to recent advances in matrix estimation literature, there exist algorithms that provide data-driven methods to perform spectral thresholding (cf. \cite{Chatterjee15}). Finally, it is worth nothing that to the best of the authors' knowledge, the classical literature on SSA seem to be lacking finite sample analysis in the presence of noisy observations, which we do provide for our algorithm.

\vspace{.05in}
\noindent {\em Multiple time series viewed as matrix.} 
In a recent line of work \cite{amjad2017censored, yu2016temporal, xie16missingtimeseries, rallapalli2010exploiting, Chen05nonnegativematrix, amjad2017robust}, multiple time series have been viewed as a matrix with the primary goal 
of imputing missing values or de-noising them. Some of these works also require prior 
model assumptions on the underlying time series. For example in \cite{yu2016temporal}, as stated in 
Section \ref{sec: Introduction}, the second step of their algorithm changes based on the user's belief 
in the model that generated the data along with the multiple time series requirement.  

In summary, to the best of our knowledge, ours is the first work to give rigorous theoretical guarantees for a matrix estimation inspired algorithm for a \textit{single, univariate} time series. 

\smallskip
\noindent {\bf Recovering the hidden state.} The question of recovering the hidden state from noisy observations is quite prevalent and
a workhorse of classical systems theory. For example, most of the system identification literature focuses on recovering model parameters of a Hidden Markov Model. While Expectation-Maximization or Baum-Welch are the go-to approaches, there is limited
theoretical understanding of it in generality (for example, see a recent work \cite{yang2017statistical} for an overview) and knowledge of the underlying model is required. For instance, \cite{bertsimas1999estimation} proposed an optimization based, statistically consistent estimation method. However, the optimization ``objective'' encoded knowledge of the precise underlying model.

It is worth comparing our method with a recent work \cite{amjad2017censored} where the authors attempt to recover the hidden time-varying parameter of a Poisson process via matrix estimation. Unlike our work, they require access to multiple time series. In essence, our algorithm provides the solution to the same question {\em without} requiring access to any other time series!
%
%
%

\subsection{Notation} \label{sec:notation}

For any positive integer $N$, let $[N] = \{1, \dots, N\}$. For any vector $v \in \mathbb{R}^n$, we denote its Euclidean ($\ell_2$) norm by $\norm{v}_2$, and define $\norm{v}_2^2 = \sum_{i=1}^n v_i^2$. In general, the $\ell_p$ norm for a vector $v$ is defined as $\norm{v}_p = \Big( \sum_{i=1}^n \abs{v_i}^p \Big)^{1/p}$. 

For a $m \times n$ real-valued matrix $\bA = [A_{ij}]$, its spectral/operator norm, denoted by $\norm{\bA}$, is defined as $\norm{\bA}_2 = \max_{1 \le i \le k} \abs{\sigma_i}$,
where $k = \min\{m, n\}$ and $\sigma_i$ are the singular values of $\bA$ (assumed to be in decreasing order and repeated by multiplicities). The Frobenius norm, also known as the Hilbert-Schmidt norm, is defined as 
$\norm{\bA}_F^2 = \sum_{i=1}^m \sum_{j=1}^n A_{ij}^2 ~= \sum_{i=1}^k \sigma_i^2.$
The max-norm, or sup-norm, is defined as
$\norm{\bA}_{\max} = \max_{i, j} \abs{A_{ij}}$.
The Moore-Penrose pseudoinverse $\bA^{\dagger}$ of $\bA$ is defined as 
\[
	\bA^{\dagger}  = \sum_{i=1}^k (1/ \sigma_i) y_i x_i^T, \quad \text{where} \quad \bA = \sum_{i=1}^k \sigma_i x_i y_i^T,
\]
with $x_i$ and $y_i$ being the left and right singular vectors of $\bA$, respectively. 

For a random variable $X$ we define its sub-gaussian norm as 
\begin{align*} 
	\norm{X}_{\psi_2} = \text{inf} \Big \{ t > 0: \Ex \exp(X^2 /t^2) \le 2 \Big \}.
\end{align*}
If $\norm{X}_{\psi_2}$ is bounded by a constant, we call $X$ a sub-gaussian random variable.

Let $f$ and $g$ be two functions defined on the same space. We say that $f(x) = O(g(x))$ if and only if there exists a positive real number $M$ and a real number $x_0$ such that for all $x \ge x_0$, $\abs{f(x)} \le M \abs{g(x)}$. Similarly, we say $f(x) = \Omega(g(x))$ if and only if for all $x \ge x_0$, $\abs{f(x)} \ge M \abs{g(x)}$. 

\subsection{Organization}

In Section \ref{sec:me}, we list the desired properties needed from a matrix estimation estimation method in order to achieve our theoretical guarantees for imputation and prediction. In Section \ref{sec:algorithm}, we formally describe the matrix estimation based algorithms we utilize for time series analysis. In Section \ref{sec:main_results}, we identify the required properties of time series models $f$ under which we can provide finite sample analysis for imputation and prediction performance. In Section \ref{sec:model}, we list a broad set of time series models that satisfy the properties in Section \ref{sec:main_results}, and we analyze the sample complexity of our algorithm for each of these models. Lastly, in Section \ref{sec:experiments_appendix}, we corroborate our theoretical findings with detailed experiments.

\section{Matrix Estimation}\label{sec:me}

\subsection{Problem setup}
Consider an $m \times n$ matrix $\bM$ of interest. Suppose we observe a random subset of the entries of a noisy signal matrix $\bX$, such that $\Ex[\bX] = \bM$. For each $i \in [m]$ and $j \in [n]$, the $(i,j)$-th entry $X_{ij}$ is a random variable that is observed with probability $p \in (0,1]$ and is missing with probability $1-p$, independently of all other entries. Given $\bX$, the goal is to produce an estimator $\bhM$ that is ``close'' to $\bM$. We use two metrics to quantify the estimation error: 

\vspace{2mm}
\noindent (1) mean-squared error, 
\begin{align} \label{eq:matrix_mse}
	\text{MSE}(\bhM, \bM) := \Ex \Big[ \frac{1}{mn} \sum_{i=1}^m \sum_{j=1}^n (\hat{M}_{ij} - M_{ij})^2 \Big];
\end{align}
(2) max row sum error, 
\begin{align} \label{eq:mrse}
	\text{MRSE}(\bhM, \bM) :=  \Ex \Big[ \frac{1}{\sqrt{n}} \max\limits_{i \in [m]} \Big(\sum_{j = 1}^n  ( \hat{M}_{ij} - M_{ij} )^2\Big)^{1/2} \Big]. 
\end{align}
Here, $\hat{M}_{ij}$ and $M_{ij}$ denote the $(i,j)$-th elements of $\bhM$ and $\bM$, respectively. We highlight that the MRSE is a non-standard matrix estimation error metric, but we note that it is a stronger notion than the $\text{RMSE}(\bhM, \bM)$\footnote{$\text{RMSE}(\bhM, \bM) := \Ex \Big[ \frac{1}{\sqrt{mn}} \Big(\sum_{i=1}^m \sum_{j=1}^n (\hat{M}_{ij} - M_{ij})^2 \Big)^{1/2} \Big]$.}; in particular, it is easily seen that $\text{MRSE}(\bhM, \bM) \ge  \text{RMSE}(\bhM, \bM) $. Hence, for any results we prove in Section \ref{sec:main_results} regarding the $\text{MRSE}$, any known lower bounds for $\text{RMSE}$ of matrix estimation algorithms immediately hold for our results. We now give a definition of a matrix estimation algorithm, which will be used in the following sections.

\begin{definition}\label{def:matrix_estimation}
A matrix estimation algorithm, denoted as $\text{ME}: \mathbb{R}^{m \times n} \rightarrow \mathbb{R}^{m \times n}$, takes as input a noisy matrix $\bX$ and outputs an estimator $\bhM$.
\end{definition}

\subsection{Required properties of matrix estimation algorithms}\label{sec:req_prop_matrix_est}

As aforementioned, our algorithm (Section \ref{sec:algorithm_desc}) utilizes matrix estimation as a pivotal ``blackbox'' subroutine, which enables accurate imputation and prediction in a model and noise agnostic setting. Over the past decade, the field of matrix estimation has spurred tremendous theoretical and empirical research interest, leading to the emergence of a myriad of algorithms including spectral, convex optimization, and nearest neighbor based approaches. Consequently, as the field continues to advance, our algorithm will continue to improve in parallel. We now state the properties needed of a matrix estimation algorithm $\text{ME}(\cdot)$ to achieve our theoretical guarantees (formalized through Theorems \ref{thm:imputation} and \ref{thm:asymptotics}); refer to Section \ref{sec:notation} for matrix norm definitions. 

\begin{property}\label{prop:2.1}
Let \text{ME} satisfy the following: Define $\bY = [Y_{ij}]$ where $Y_{ij} = X_{ij}$ if $X_{ij}$ is observed, and $Y_{ij} = 0$ otherwise. Then, for all $p \ge \max(m,n)^{-1 + \zeta}$ and some $\zeta \in (0,1)$, the produced estimator $\bhM = \text{ME}(\bX)$ satisfies
\begin{align}
	\norm{\hat{p} \bhM - p \bM}^2_F &\le  \frac{1}{mn} \, C_1\,  \norm{\bY - p \bM}  \, \norm{p \bM}_*.
\end{align}
Here, $\hat{p}$ 
\footnote{Precisely, we define $\hat{p} = \max \{ \frac{1}{mn} \sum_{i=1}^m \sum_{j=1}^n \mathds{1}_{X_{ij} \text{ observed}}, \frac{1}{mn} \}$.} 
denotes the proportion of observed entries in $\bX$ and $C_1$ is a universal constant.
\end{property}

\noindent	We argue the two quantities in Property \ref{prop:2.1}, $\norm{\bY -  p \bM}$ and $\norm{\bM}_*$, are natural. $\norm{\bY -  p \bM}$ quantifies the amount of noise corruption on the underlying signal matrix $\bM$; for many settings, this norm concentrates well (e.g., a matrix with independent zero-mean sub-gaussian entries scales as $\sqrt{m} + \sqrt{n}$ with high probability \cite{vershynin2010introduction}). $\norm{\bM}_*$ quantifies the inherent model complexity of the latent signal matrix; this norm is well behaved for an array of situations, including low-rank and Lipschitz matrices (e.g., for low-rank matrices, $\norm{\bM}_*$ scales as $\sqrt{rmn}$ where r is the rank of the matrix, see \cite{Chatterjee15} for bounds on $\norm{\bM}_*$ under various settings). 
We note the universal singular value thresholding algorithm proposed in \cite{Chatterjee15} is one such algorithm that satisfies Property \ref{prop:2.1}. We provide more intuition for why we choose Property \ref{prop:2.1} for our matrix estimation methods in Section \ref{sec:imputation}, where we bound the imputation error.

\begin{property} \label{prop:2.2}
Let \text{ME} satisfy the following: For all $p \ge p^*(m,n)$, the produced estimator $\bhM = \text{ME}(\bX)$ satisfies
\begin{align}
	\emph{MRSE}(\bhM, \bM) &\le \delta_3(m,n)
\end{align}
where $\lim_{m,n \to \infty} \delta_3(m,n) = 0$.
\end{property}
\noindent Property \ref{prop:2.2} requires the normalized max row sum error to decay to zero as we collect more data. While spectral thresholding and convex optimization methods accurately bound the average mean-squared error, minimizing norms akin to the normalized max row sum error require matrix estimation methods to utilize ``local" information, e.g., nearest neighbor type methods. For instance, \cite{ZhangLevinaZhu15} satisfies Property \ref{prop:2.2} for generic latent variable models (which include low-rank models) with $p^*(m,n) = 1$; \cite{LeeLiShahSong16} also satisfies Property \ref{prop:2.2} for 
$p^*(m,n) \gg \min(m,n)^{-1/2}$; \cite{BorgsChayesLeeShah17} establishes this for low-rank models as long as $p^*(m,n) \gg \min(m,n)^{-1}$.

\section{Algorithm} \label{sec:algorithm}

\subsection{Notations and definitions} \label{sec:algo_notation}
Recall that $X(t)$ denotes the observation at time $t \in [T]$ where $\Ex[X(t)] = f(t)$. We shall use the notation $X[s:t] = [X(s),\dots, X(t)]$ for any $s \leq t$. Furthermore, we define $L > 1$ to be an algorithmic hyperparameter and $N =  \lfloor T/ L \rfloor - 1$.
For any $L \times N$ matrix $\bA$, let $A_L = [A_{Lj}]_{j \le N}$ represent the the last row of $\bA$. Moreover, let $\btA = [A_{ij}]_{i < L, j \le N}$ denote the $(L-1) \times N$ submatrix obtained by removing the last row of $\bA$. 

\subsection{Viewing a univariate time series as a matrix.} \label{sec:ts_2_matrix}

We begin by introducing the crucial step of transforming a single, univariate time series into the corresponding Page matrix. Given time series data $X[1:T]$, we construct $L$ different $L \times N$ matrices $\bX^{(k)}$ defined as
\begin{align}
	\bX^{(k)}  &= [X_{ij}^{(k)}] = [X(i + (j-1)L + (k-1))]_{i \le L, j \le N},
\end{align}
where $k \in [L]$\footnote{Technically, to define each $\bX^{(k)}$, we need access to $T' = T + L$ time steps of data. To reduce notational overload and since it has no bearing on our theoretical analysis, we let $T' = T$.}.
In words, $\bX^{(k)}$ is obtained by dividing the time series into $N$ non-overlapping contiguous intervals each of length $L$, thus constructing $N$ columns; for each $k \in [L]$, $\bX^{(k)}$ is the $k$-th shifted version with starting value $X(k)$. For the purpose of imputation, we shall only utilize $\bX^{(1)}$. In the case of forecasting, however, we shall utilize $\bX^{(k)}$ for all $k \in [L]$. We define $\bM^{(k)}$ analogously to $\bX^{(k)}$ using $f(t)$ instead of $X(t)$.

\subsection{Algorithm description}\label{sec:algorithm_desc}

We will now describe the imputation and forecast algorithms separately (see Figure \ref{fig:algo}). 

\noindent \textbf{Imputation.} Due to the matrix representation $\bX^{(1)}$ of the time series, the task of imputing missing values and de-noising observed values translates to that of matrix estimation. 

	\begin{enumerate}
		
		\item Transform the data $X[1:T]$ into the matrix $\bX^{(1)}$ via the method outlined in Subsection \ref{sec:ts_2_matrix}. 
		
		\item Apply a matrix estimation method (as in Definition \ref{def:matrix_estimation}) to produce $\bhM^{(1)} = \text{ME} (\bX^{(1)})$. 
		
		\item Produce estimate: $\widehat{f}_I(i + (j-1)L) := \widehat{M}^{(1)}_{ij}$ for $i \in [L]$ and  $j \in [N]$.
	\end{enumerate}

\noindent \textbf{Forecast.} In order to forecast future values, we first de-noise and impute via the procedure outlined above, and then learn a linear relationship between the the last row and the remaining rows through linear regression.

	\begin{enumerate}
		
		\item For each $k \in [L]$, apply the imputation algorithm to produce $\bhtkM$ from $\btX^{(k)}$. 
		
		\item For each $k \in [L]$, define $\hat{\beta}^{(k)} = \argmin_{v \in \mathbb{R}^{L-1}} \norm{X^{(k)}_L - (\bhtkM)^T v}_2^2$.

		\item Produce the estimate at time $t > T$ as follows:
		\begin{itemize}
			
			\item[i)] Let $v_t = [X(t - L + 1) : X(t-1)]$ and $k = (t \mod L) + 1$. 
			\item[ii)] Define $\alpha_t = \argmin_{\alpha \in \mathbb{R}^N} \norm{ v_t - \bhtkM \alpha}_2^2$. 
			\item[iii)] Let $v^{\text{proj}}_t = \bhtkM \alpha_t$.
			\item[iv)] Produce the estimate: $\hat{f}_F(t) = (v^{\text{proj}}_t )^T\cdot \hat{\beta}^{(k)}$. 
		
		\end{itemize}
		
	\end{enumerate}
	
\noindent \textbf{Why $\bX^{(k)}$ is necessary for forecasting:} For imputation, we are attempting to de-noise all observations made up to time $T$; hence, it suffices to only use $\bX^{(1)}$ since it contains all of the relevant information. However, in the case of making predictions, we are only creating an estimator for the last row. Thus, if we take $X^{(1)}$ for instance, then it is not hard to see that our prediction algorithm only produces estimates for $X(L), X(2L), X(3L), \dots,$ and so on. Therefore, we must repeat this procedure $L$ times in order to produce an estimate for each entry. \\

\noindent \textbf{Choosing the number of rows $L$: } Theorems \ref{thm:imputation} and \ref{thm:asymptotics} (and the associated corollaries) suggest $L$ should be as large as possible with the requirement $L =o(N)$. Thus, it suffices to let $N = L^{1 + \delta}$ for any $\delta > 0$, e.g., $N = L^2 = T^{2/3}$.

\section{Main Results} \label{sec:main_results}

\subsection{Properties}\label{sec:notation_and_definitions}

We now introduce the required properties for the matrices $\bX^{(k)}$ and $\bM^{(k)}$ to identify the time series models $f$ for which our algorithm provides an effective method for imputation and prediction. Under these properties, we state Theorems \ref{thm:imputation} and \ref{thm:asymptotics}, which establish the efficacy of our algorithm. The proofs of these theorems can be found in Appendices \ref{sec:appendix:imputation} and \ref{sec:appendix:forecast}, respectively. In Section \ref{sec:model}, we argue these properties are satisfied for a large class of time series models. 

\begin{property} {\bf ($r, \delta_1$)-imputable} \label{prop.one} \\
Let matrices $\bX^{(1)}$ and $\bM^{(1)}$ satisfy the following: 
	\begin{itemize}[leftmargin=*]
		\item[]{\bf A.} For each $i \in [L]$ and $j \in [N]$: 
		\begin{itemize}
			\item[1.] $X_{ij}^{(1)}$ are independent sub-gaussian random variables\footnote{Recall that this condition only requires the per-step noise to be independent; the underlying mean time series $f$ remains highly correlated.} satisfying $\Ex[X_{ij}^{(1)}] = M_{ij}^{(1)}$ and $\norm{X_{ij}^{(1)}}_{\psi_2} \le \sigma$.
			\item[2.] $X^{(1)}_{ij}$ is observed with probability $p \in (0, 1]$, independent of other entries. 
		\end{itemize}
		
		\item[]{\bf B.} There exists a matrix $\bM_{(r)}$ of rank $r$ such that for $\delta_1 \ge 0$, 
		$$\norm{\bM^{(1)} - \bM_{(r)}}_{\max} \le \delta_1.$$
	\end{itemize}
\end{property}

\begin{property} {\bf ($C_{\beta}, \delta_2$)-forecastable} \label{prop.two} \\
For all $k \in [L]$, let matrices $\bX^{(k)}$ and $\bM^{(k)}$ satisfy the following: 
	\begin{itemize} [leftmargin=*]
		\item[]{\bf A.} For each $i \in [L]$ and $j \in [N]$: 
		\begin{itemize}
			\item[1.] $X^{(k)}_{ij} = M^{(k)}_{ij} + \epsilon_{ij}$, where $\epsilon_{ij}$ are independent sub-Gaussian random variables satisfying $\Ex[\epsilon_{ij}] = 0$ and $\text{Var}(\epsilon_{ij}) \le \sigma^2$. 		
			\item[2.] $X^{(k)}_{ij}$ is observed with probability $p \in (0, 1]$, independent of other entries. 
		\end{itemize}
		
		\item[]{\bf B.} There exists a $\beta^{*(k)} \in \mathbb{R}^{L-1}$ with $\norm{\beta^{*(k)}}_1 \le C_{\beta}$ for some constant $C_{\beta} > 0$
		and $\delta_2 \ge 0$ such that
		$$\norm{M^{(k)}_L - (\btM^{(k)})^T\beta^{*(k)}}_2 \le  \delta_2.$$
	\end{itemize}
\end{property}

\noindent For forecasting, we make the more restrictive additive noise assumption since we focus on linear forecasting methods. Such methods generally require additive noise models. If one can construct linear forecasters under less restrictive assumptions, then we should be able to lift the analysis of such a forecaster to our setting in a straightforward way.

\subsection{Imputation} \label{sec:imputation}

The imputation algorithm produces $\hat{f}_I = [\hat{f}_I(t)]_{t=1:T}$ as the estimate for the underlying time series $f = [f(t)]_{t=1:T}$. We measure the imputation error through the relative mean-squared error: 
\begin{align}
	\text{MSE}(\hat{f}_I, f) &:=   \dfrac{\Ex \, \norm{\hat{f}_I - f}_2^2}{\norm{f}_2^2}.
\end{align}

Recall from the imputation algorithm in Section \ref{sec:algorithm_desc} that $\bM^{(1)}$ is the Page matrix corresponding to $f$ and $\bhM^{(1)}$ is the estimate $\text{ME}$ produces; i.e. $\bhM^{(1)} = \text{ME} (\bX^{(1)})$. It is then easy to see that for any matrix estimation method we have
\begin{align} \label{eq:imputation_matrix}
	\text{MSE}(\hat{f}_I, f) = \dfrac{\Ex \norm{\bhM^{(1)} - \bM^{(1)}}_F^2} { \norm{\bM^{(1)}}_F^2}.
\end{align}
Thus, we can immediately translate the (un-normalized) $\text{MSE}$ of {\em any} matrix estimation method to the imputation error $\text{MSE}(\hat{f}_I, f)$ of the corresponding time series.

However, to highlight how the rank and the low-rank approximation error $\delta_1$ of the underlying mean matrix $\bM^{(1)}$ (induced by $f$) affect the error bound, we rely on Property \ref{prop:2.1}, which elucidates these dependencies through the quantity $\norm{\bM}_*$. Thus, we have the following theorem that establishes a precise link between time series imputation and matrix estimation methods.

\begin{thm}\label{thm:imputation}
Assume Property \ref{prop.one} holds and $\text{ME}$ satisfies Property \ref{prop:2.1}. Then for some $C_1, C_2, C_3, c_4 > 0$, 
\begin{align}
	\emph{MSE}(\hat{f}_I, f) \le \frac{C_1\sigma}{p}\Bigg(\frac{L N \delta_1}{\norm{f}_2^2} + \frac{\sqrt{rL} N \delta_1}{\norm{f}_2^2} + \frac{\sqrt{r N}}{\norm{f}_2}\Bigg) + \dfrac{C_2(1-p)}{pLN} + C_{3} e^{-c_{4} N}.
\end{align}
\end{thm}

\noindent Theorem \ref{thm:imputation} states that any matrix estimation subroutine $\text{ME}$ that satisfies Property \ref{prop:2.1} will accurately filter noisy observations and recover missing values. This is achieved provided that the rank of $\bM_{(r)}$ and our low-rank approximation error $\delta_1$ are not too large. Note that knowledge of $r$ is not required apriori for many standard matrix estimation algorithms. For instance, \cite{Chatterjee15} does not utilize the rank of $\bM$ in its estimation procedure; instead, it performs spectral thresholding of the observed data matrix in an adaptive, data-driven manner. Theorem \ref{thm:imputation} implies the following consistency property of $\hat{f}_I$.

\begin{cor} \label{corollary:imputation} 
Let the conditions for Theorem \ref{thm:imputation} hold. Let $\norm{f}_2^2 = \Omega(T)$ 
\footnote{Note the condition $\norm{f}_2^2 = \Omega(T)$ is easily satisfied for any time series $f$ by adding a constant shift to every observation $f(t)$.}.
Further, suppose $f$ is ($C_5 L^{1-\epsilon_2}, C_6 L^{-\epsilon_1}$)-imputable for some $\epsilon_1, \epsilon_2 \in (0,1)$ and $C_5, C_6 > 0$. 
Then for $p \gg L^{-\min\big(2 \epsilon_1, \epsilon_2\big)}$
\[
\lim_{T \to\infty} \emph{MSE}(\hat{f}_I, f) = 0.
\]
\end{cor}

\noindent We note that Theorem \ref{thm:imputation} follows in a straightforward manner from Property \ref{prop:2.1} and standard results from random matrix theory \cite{vershynin2010introduction}. However, we again highlight that our key contribution lies in establishing that the conditions of Corollary \ref{corollary:imputation} hold for a large class of time series models (Section \ref{sec:model}).


\subsection{Forecast} \label{sec:forecasting}

Recall $\hat{f}_F(t)$ can only utilize information until time $t-1$. For all $k \in [L]$, our forecasting algorithm learns $\hat{\beta}^{(k)}$ with the previous $L-1$ time steps. We measure the forecasting error through: 
\begin{align}
	\text{MSE}(\hat{f}_F, f) := \dfrac{1}{T-L+1} \Ex \norm{ \hat{f}_F - f}_2^2. 
\end{align}
Here, $\hat{f}_F = [\hat{f}_F(t)]_{t=L:T}$ denotes the vector of forecasted values. The following result relies on a novel analysis of how applying a matrix estimation pre-processing step affects the prediction error of error-in-variable regression problems (in particular, it requires analyzing a non-standard error metric, the $\text{MRSE}$). 

\begin{thm} \label{thm:asymptotics}
Assume Property \ref{prop.two} holds and $\text{ME}$ satisfies Property \ref{prop:2.2}, with $p \ge p^*(L, N)$\footnote{Refer to Section \ref{sec:req_prop_matrix_est} for lower bounds on $p^*(L, N)$ for various $\text{ME}$ algorithms. The dependence of the bound on $p$ is implicitly captured in $\delta_3$.}. Let $\hat{r} := \max\limits_{k \in [L]} \emph{rank}(\bhtkM)$. Then, 
\[
\emph{MSE}(\hat{f}_F, f) \le \frac{1}{N-1} \Big( ( \delta_2 + \sqrt{C_\beta N}  \delta_3 )^2 + 2 \sigma^2 \hat{r} \Big).
\]
\end{thm}
\noindent Note that $\hat{r}$ is trivially bounded by $L = o(N)$ by assumption (see Section \ref{sec:algorithm}). If the underlying matrix $\bM$ is low-rank, then $\text{ME}$ algorithms such as the USVT algorithm (cf. \cite{Chatterjee15}) will output an estimator with a small $\hat{r}$. However, since our bound holds for general $\text{ME}$ methods, we explicitly state the dependence on $\hat{r}$. 

In essence, Theorem \ref{thm:asymptotics} states that any matrix estimation subroutine $\text{ME}$ that satisfies Property \ref{prop:2.2} will produce accurate forecasts from noisy, missing data. This is achieved provided the linear model approximation error $\delta_2$ is not too large (recall $\delta_3 = o(1)$ by Property $\ref{prop:2.2}$). Additionally, Theorem \ref{thm:asymptotics} implies the following consistency property of $\hat{f}_F$.

\begin{cor} \label{cor:asymptotics}
Let the conditions for Theorem \ref{thm:asymptotics} hold. Suppose $f$ is $(C_1, C_2 \sqrt{N} L^{-\epsilon_1})$-forecastable for any $\epsilon_1, C_1, C_2 > 0$ and $N = L^{1 + \delta}$ for any $\ \delta > 0$. Then for $p \ge p^*(L, N)$, such that $\lim_{L, N \to \infty} \delta_3(L, N) = 0$ for $p^*(L, N)$,
\[
\lim_{T \to\infty} \emph{MSE}(\hat{f}_F, f) = 0.
\]
\end{cor}


\noindent Similar to the case of imputation, a large contribution of this work is in establishing that the conditions of Corollary \ref{cor:asymptotics} hold for a large class of time series models (Section \ref{sec:model}). Effectively, Corollary \ref{cor:asymptotics} demonstrates that learning a simple linear relationship among the singular vectors of the de-noised matrix is sufficient to drive the empirical error to zero for a broad class of time series models. The simplicity of this linear method suggests that our estimator will have low generalization error, but we leave that as future work. 

We should also note that for auto-regressive processes (i.e., $f(t) = \sum_{g=1}^G \alpha_g f(t-1) + \epsilon(t)$ where $\epsilon(t)$ is mean zero noise), previous works (e.g., \cite{nardi2011autoregressive}) have already shown that simple linear forecasters are consistent estimators. For such models, it is easy to see that the underling mean matrix $\bM^{(k)}$ is not (approximately) low-rank, and so it is not necessary to pre-process the data matrix via a matrix estimation subroutine as we propose in Section \ref{sec:algorithm_desc}.

\section{Family of Time Series That Fit Our Framework} \label{sec:model}

In this section, we list out a broad set of time series models that satisfy Properties \ref{prop.one} and \ref{prop.two}, which are
required for the results stated in Section \ref{sec:main_results}. The proofs of these results can be found in Appendix \ref{sec:appendix:models}. To that end, we shall repeatedly use the following model types for our observations. 
\begin{itemize} [wide, labelwidth=!, labelindent=0pt]
\item[] {\em Model Type 1.} For any $t \in \mathbb{Z}$, 
let $X(t)$ be a sequence of independent sub-gaussian random variables
with $\mathbb{E}[X(t)] = f(t)$ and $\norm{X(t)}_{\psi_2} \leq \sigma$. Note the noise on $f(t)$ is generic (e.g., non-additive).

\item[] {\em Model Type 2.} For $t \in \mathbb{Z}$, let $X(t) = f(t) + \epsilon(t)$ 
where $\epsilon(t)$ are independent sub-gaussian random variables with $\Ex[\epsilon(t)] = 0$ and $\text{Var}(\epsilon(t)) \le \sigma^2$. 
\end{itemize}

\subsection{Linear recurrent functions (LRFs)}
For $t \in \mathbb{Z}$, let 
\begin{equation}\label{eq:LTI_representation}
f^{\text{LRF}}(t)  = \sum_{g=1}^{G} \alpha_g f(t-g).
\end{equation}

\begin{prop}\label{prop:lowrank_LTI} {\color{white}.}
\begin{itemize}
	\item[(i)] Under {\em Model Type 1}, $f^{\emph{LRF}}$ satisfies Property \ref{prop.one} with $\delta_1 = 0$ and $r = G$\footnote{To see this, take $G = 2$ for example. WLOG, let us consider the first column. Then $f(3) = f(2) + f(1)$, which in turn gives $f(4) = f(3) + f(2) = 2f(2) + f(1)$ and $f(5) = f(4) + f(3) = 3f(2) + 2f(1)$. By induction, it is not hard to see that this holds more generally for any finite $G$.}. 
	\item[(ii)] Under {\em Model Type 2}, $f^{\emph{LRF}}$ satisfies Property \ref{prop.two} with $\delta_2 = 0$ and $C_{\beta} = C G$ for all $k \in [L]$ where $C > 0$ is an absolute constant.
\end{itemize}
\end{prop} 
\noindent By Proposition \ref{prop:lowrank_LTI}, Theorems \ref{thm:imputation} and \ref{thm:asymptotics} give the following corollaries:

\begin{cor}\label{corollary:lowrank_LTI_imputation}
Under {\em Model Type 1}, let the conditions of Theorem \ref{thm:imputation} hold. Let $N = L^{1 + \delta}$ for any $\ \delta > 0$. Then for some $C > 0$, if 
\[ 
T \ge C \cdot \Bigg(\frac{G}{\delta_{\text{error}}^2}\Bigg)^{2 + \delta},
\]
we have $\emph{MSE}(\hat{f}_I, f^{\emph{LRF}}) \le \delta_{\text{error}}$.
\end{cor}


\begin{cor}\label{corollary:lowrank_LTI_forecasting}
Under {\em Model Type 2}, let the conditions of Theorem \ref{thm:asymptotics} hold. Let $N = L^{1 + \delta}$ for any $\delta > 0$. Then for some $C > 0$, if 
\[ 
T \ge C \cdot \Bigg( \frac{\sigma^2 }{\delta_{\text{error}} - G\delta_3^2} \Bigg)^{\frac{2+\delta}{\delta}},
\]
we have $\emph{MSE}(\hat{f}_F, f^{\emph{LRF}})  \le \delta_{\text{error}}$.
\end{cor}

\noindent We now provide the rank $G$ of an important class of time series methods---a finite sum of the product of polynomials, harmonics, and exponential time series functions.

\begin{prop}\label{prop:lowrank_LTI_example}
Let $P_{m_a}$ be a polynomial of degree $m_a$. Then,
\[
f(t)= \sum_{a=1}^A \exp{\alpha_a t} \cos(2\pi \omega_a t + \phi_a) P_{m_a}(t)
\]
admits a representation as in \eqref{eq:LTI_representation}. Further the order $\ G$ of $f(t)$ is independent of $\ T$, the number of observations, and is bounded by
\[ 
G \le A(m_{\max} + 1)(m_{\max} + 2)
\]
where $m_{\max} = \max_{a \in A} m_a$.
\end{prop}

\subsection{Functions with compact support}
For $t \in \mathbb{Z}$, let 
\begin{equation}\label{eq:harmonics_representation}
f^{\text{Compact}}(t) = g(\varphi(t))
\end{equation}
where $\varphi: \mathbb{Z} \to [-C_1, C_1]$ takes the form $\varphi(t + s) = \sum_{l=1}^{G} \alpha_l a_l(t) b_l(s)$ with $\alpha_l \in [-C_2, C_2], a_l: \mathbb{Z} \to [0,1], b_l: \mathbb{Z} \to [0,1]$; and $g: [-C_1, C_1] \to \mathbb{R}$ is $\mathcal{L}$-Lipschitz for some $C_1, C_2 > 0$.

\begin{prop}\label{prop:lowrank_harmonics}
For any $\epsilon \in  (0, 1)$,
\begin{itemize}
	\item[(i)] Under {\em Model Type 1}, $f^{\emph{Compact}}$ satisfies Property \ref{prop.one} with $\delta_1 = \frac{ C \mathcal{L}}{L^{\epsilon}}$ and 
	$r = L^{G\epsilon}$ for some $C > 0$.	
	\item[(ii)] Under {\em Model Type 2}, $f^{\emph{Compact}}$ satisfies Property \ref{prop.two} with $\delta_2 = 2 \delta_1 \sqrt{N}$ and $C_{\beta} = 1$ for all $k \in [L]$.
\end{itemize}
\end{prop}
\noindent Using Proposition \ref{prop:lowrank_harmonics}, Theorems \ref{thm:imputation} and \ref{thm:asymptotics} immediately lead to the following corollaries.

\begin{cor}\label{corollary:lowrank_harmonics_imputation}
Under {\em Model Type 1}, let the conditions of Theorem \ref{thm:imputation} hold. Let $N = L^{1 + \delta}$ for any $\delta > 0$. Then for some $\ C > 0$ and any $\epsilon \in (0, 1)$, if 
\[ 
T \ge C \Bigg( \Big(\dfrac{1}{\delta_{\text{error}}}\Big)^{\frac{2}{1-G\epsilon}} + \Big(\frac{ \mathcal{L}}{\delta_{\text{error}}}\Big)^{\frac{1}{\epsilon}} \Bigg)^{2 + \delta},
\]
we have $\emph{MSE}(\hat{f}_I, f^{\emph{LRF}}) \le \delta_{\text{error}}$.
\end{cor}


\begin{cor}\label{corollary:lowrank_harmonics_forecasting}
Under {\em Model Type 2}, let the conditions of Theorem \ref{thm:asymptotics} hold. Let $N = L^{1 + \delta}$ for any $\delta > 0$. Then for some $\ C > 0$ and any $\epsilon \in (0, 1)$, if 
\[ 
T \ge C \Bigg( \frac{\sigma^2}{\delta_{\text{error}} - \Big(\frac{ \mathcal{L}}{L^{\epsilon}} + \delta_3\Big)^{2}} \Bigg)^{\frac{2+\delta}{\delta}},
\]
we have $\emph{MSE}(\hat{f}_F, f^{\emph{LRF}})  \le \delta_{\text{error}}$.
\end{cor}

\noindent As the following proposition will make precise, any Lipschitz function of a periodic time series falls into this family.

\begin{prop}\label{prop:lowrank_harmonics_example}
Let
\begin{align}
f^{\emph{Harmonic}}(t)= \sum_{r=1}^R \varphi_r \Big(\sin(2 \pi \omega_r t + \phi) \Big),
\end{align}
where $\varphi_r$ is $\mathcal{L}_r$-Lipschitz and $\omega_r$ is rational, admits a representation as in \eqref{eq:harmonics_representation}. Let $x_{\lcm}$ denote the fundamental period.\footnote{The ``fundamental period'', $x_{\lcm}$, of $\{\omega_1, \dots, \omega_G\}$ is the smallest value such that $x_{\lcm} / (q_a/p_a)$ is an integer for all $a \in A$. 
Let $S \equiv \{q_a / p_a : g \in G\}$ and let $p_{\lcm}$ be the least common multiple (LCM) of $\{p_1, \dots, p_G\}$. Rewriting $S$ as $\Big\{\dfrac{q_1 * p_{\lcm} / p_1}{p_{\lcm}}, \dots, \dfrac{q_G * p_{\lcm} / p_G}{p_{\lcm}}\Big\}$, we have the set of numerators, $\{q_1 * p_{\lcm} / p_1, \dots, q_G * p_{\lcm}/ p_A\}$ are all integers and we define their LCM as $d_{\lcm}$. It is easy to verify that $x_{\lcm} = d_{\lcm}/p_{\lcm}$ is indeed a fundamental period. As an example, consider $x = \{n, n/2, n/3, \dots, n/n-1\}$, in which case the above computation results in $x_{\lcm} = n$.}
Then the Lipschitz constant $\mathcal{L}$ of $f^{\emph{Harmonic}}(t)$ is bounded by 
\[
\mathcal{L} \le 2\pi \cdot \max_{r \in R}(\mathcal{L}_r) \cdot \max_{r \in R}(\omega_r) \cdot x_{\lcm}.
\] 
\end{prop}

\subsection{Finite sum of sublinear trends}
Consider $f^{\text{Trend}}(t)$  such that
\begin{equation}\label{eq:sublinear_representation}
\abs{\frac{d f^{\emph{Trend}}(t)}{dt}} \le C_*t^{-\alpha} 
\end{equation}
for some $\alpha, C_* > 0$.

\begin{prop}\label{prop:lowrank_sublinear}
Let $\abs{\frac{d f^{\emph{Trend}}(t)}{dt}} \le C_*t^{-\alpha}$ for some $\alpha, C_* > 0$. 
Then for any $\epsilon \in (0, \alpha)$, 
\begin{itemize}
	\item[(i)] Under {\em Model Type 1}, $f^{\emph{Trend}}$ satisfies Property \ref{prop.one} with $\delta_1 = \frac{C_*}{L^{\epsilon/2}}$ and $r = L^{\epsilon/\alpha} + \frac{L - L^{\epsilon/\alpha}}{L^{\epsilon/2}}$.
	
	\item[(ii)] Under {\em Model Type 2}, $f^{\emph{Trend}}$ satisfies Property \ref{prop.two} with $\delta_2 = 2 \delta_1 \sqrt{N}$ and $C_{\beta} = 1$ for all $k \in [L]$.
\end{itemize}
\end{prop} 

\noindent By Proposition \ref{prop:lowrank_sublinear} and Theorems \ref{thm:imputation} and \ref{thm:asymptotics}, we immediately have the following corollaries on the finite sample performance guarantees of our estimators.

\begin{cor}\label{corollary:lowrank_sublinear_imputation}
Under {\em Model Type 1}, let the conditions of Theorem \ref{thm:imputation} hold. Let $N = L^{1 + \delta}$ for any $\delta > 0$. Then for some $\ C > 0$, if 
\[ 
T \ge C \cdot \Bigg(\frac{1}{\delta_{\text{error}}^{2(\alpha + 1)/\alpha}} \Bigg)^{2 + \delta},
\]
we have $\emph{MSE}(\hat{f}_I, f^{\emph{LRF}}) \le \delta_{\text{error}}$.
\end{cor}


\begin{cor}\label{corollary:lowrank_sublinear_forecasting}
Under {\em Model Type 2}, let the conditions of Theorem \ref{thm:asymptotics} hold. Let $N = L^{1 + \delta}$ for any $\delta > 0$. Then for some $\ C > 0$ and for any $\epsilon \in (0, \alpha)$, if 
\[ 
T \ge C \cdot \Bigg( \frac{\sigma^2}{\delta_{\text{error}} - (L^{-\epsilon/2} + \delta_3)^{2}} \Bigg)^{\frac{2+\delta}{\delta}},
\]
we have $\emph{MSE}(\hat{f}_F, f^{\emph{LRF}})  \le \delta_{\text{error}}$.
\end{cor}

\begin{prop}\label{prop:lowrank_sublinear_example}
For $t \in \mathbb{Z}$ with $\alpha_b < 1$ for $b \in [B]$,
\begin{align}
f^{\text{Trend}}(t) = \sum_{b=1}^B \gamma_b t^{\alpha_b} + \sum_{q=1}^{Q}\log(\gamma_q t)
\end{align}
admits a representation as in \eqref{eq:sublinear_representation}. 
\end{prop}

\subsection{Additive mixture of dynamics}\label{mixture - harmonics and trends}
We now show that the imputation results hold even when we consider an additive mixture 
of any of the models described above. For $t \in \mathbb{Z}$, let 
\begin{align}
f^{\text{Mixture}}(t)  = \sum_{q=1}^{Q} \rho_q f_q(t).
\end{align}
Here, each $f_q$ is such that under {\em Model Type 1} with $\mathbb{E}[X(t)] = f_q(t)$,
Property \ref{prop.one} is satisfied with $\delta_1 = \delta_q$ and $r = r_q$ for $q \in [Q]$.
\begin{prop}\label{prop:lowrank_additive_mixture}
	Under {\em Model Type 1}, $f^{\emph{Mixture}}$ satisfies Property \ref{prop.one} with $\delta_1 = \sum_{q=1}^{Q} \rho_q \delta_q$  and  $r = \sum_{q=1}^{Q} r_q$.
\end{prop} 

\noindent Proposition \ref{prop:lowrank_additive_mixture} and Corollary \ref{corollary:imputation} imply the following.

\begin{cor}\label{corollary:lowrank_additive_mixture_imputation}
Under {\em Model Type 1}, let the conditions of Theorem \ref{thm:imputation} hold. For each $q \in [Q]$, let $\delta_q \le C'_q L^{-\epsilon_q}$ and $r_q = o(L)$ for some $\epsilon_q, C'_q > 0$. Then, $\lim_{T \to\infty} \emph{MSE}(\hat{f}_I, f^{\emph{Mixture}}) = 0$.
\end{cor}

\noindent{\em In summary, Corollaries \ref{corollary:lowrank_LTI_imputation}, \ref{corollary:lowrank_harmonics_imputation}, \ref{corollary:lowrank_sublinear_imputation} and \ref{corollary:lowrank_additive_mixture_imputation} 
imply that for any additive mixture of time series dynamics coming from $f^{\emph{LRF}}$, $f^{\emph{Compact}}$, and $f^{\emph{Trend}}$, the algorithm in Section \ref{sec:ts_2_matrix} produces a consistent estimator for an appropriate choice of $L$.}

\subsection{Hidden State} \label{sec:hmm}

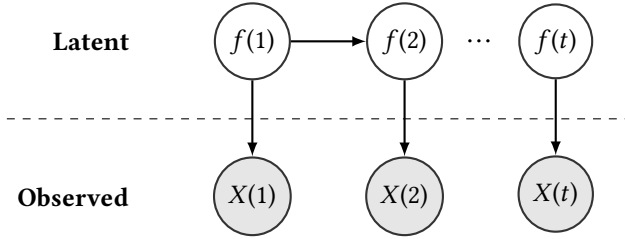
\begin{figure} [H]
\begin{center}
\begin{tikzpicture}{center}
  \node[box,draw=white!100] (Latent) {\textbf{Latent}};
  \node[main] (L1) [right=of Latent] {$f(1)$};
  \node[main] (L2) [right=of L1] {$f(2)$};
  \node[main] (Lt) [right=of L2] {$f(t)$};
  \node[main,fill=black!10] (O1) [below=of L1] {$X(1)$};
  \node[main,fill=black!10] (O2) [below=of L2] {$X(2)$};
  \node[main,fill=black!10] (Ot) [below=of Lt] {$X(t)$};
  \node[box,draw=white!100,left=of O1] (Observed) {\textbf{Observed}};
  \path (L2) -- node[auto=false]{\ldots} (Lt);
  \path (L1) edge [connect] (L2)
        (L2) -- node[auto=false]{\ldots} (Lt);
  \path (L1) edge [connect] (O1);
  \path (L2) edge [connect] (O2);
  \path (Lt) edge [connect] (Ot);
  \draw [dashed, shorten >=-1cm, shorten <=-1cm]
      ($(Latent)!0.5!(Observed)$) coordinate (a) -- ($(Lt)!(a)!(Ot)$);
\end{tikzpicture}
\end{center}
\caption{Hidden State Model with $\mathbb{E}[X(t)] = f(t)$ and $\norm{X(t)}_{\psi_2} \le \sigma$.}
\label{fig:hmms}
\end{figure}

A common problem of interest is to uncover the hidden dynamics of latent variables given noisy observations. For example, consider the problem of estimating the true weekly demand rate of umbrellas at a retail store given its weekly sales of umbrellas. This can be mathematically described as uncovering the underlying parameters of a time varying truncated Poisson process
\footnote{Recall that a {\it truncated} Poisson random variable $Y(t)$ is defined as $Y(t) = \min\{X(t), C\}$, where $C$ denotes a positive, bounded constant and $X(t) = \text{Poisson}(f(t))$.} 
whose samples are the weekly sales reports, (cf. \cite{amjad2017censored}). 
In general, previous methods to learn the hidden states either require multiple time series as inputs or require that the underlying noise model is known (refer to Section \ref{sec:related_works} for a detailed overview).

In contrast, by viewing $f(t)$ as the time-varying latent variables (see Figure \ref{fig:hmms}), we are well equipped to handle more generic noise distributions and complicated hidden dynamics. Specifically, our imputation and forecast algorithms can uncover the latent dynamics if: (i) per-step noise is sub-gaussian (additive noise is needed for forecasting); (ii) $\mathbb{E}[X(t)] = f(t)$. Moreover, our algorithm is model and noise agnostic, robust to missing entries, and comes with strong theoretical consistency guarantees (Theorems \ref{thm:imputation} and \ref{thm:asymptotics}). Given these findings, our approach is likely to become a useful gadget in the toolkit for dealing with scenarios pertinent to uncovering latent states a la Hidden Markov-like models. We corroborate our findings through experiments in Section \ref{sec:experiments_appendix}.

\subsection{Sample complexity} \label{ssec:sc}

\newcommand{\n}{n}

As discussed, our algorithm operates for a large class of models---it is not tailored for a specific model class (e.g., sum of harmonics). In particular, for a variety of model classes, our algorithm provides consistent estimation for imputation while the forecasting MSE scales with the quality of the matrix estimation algorithm $\delta_3$. Naturally, it is expected that to achieve accurate performance, the number of samples $T$ required will scale relatively poorly compared to model specific optimal algorithms. Corollaries \ref{corollary:lowrank_LTI_imputation} - \ref{corollary:lowrank_sublinear_forecasting} provide finite sample analysis that quantifies this ``performance loss'' and indicates that this loss is minor. As an example, consider imputation for any periodic time series with periods between $[n]$. By proposition \ref{prop:lowrank_LTI_example}, it is easy to see that the order $G$ of such a time series is $2n$. Thus, corollary \ref{corollary:lowrank_LTI_imputation} indicates that the MSE decays to $0$ with $T \sim \n^{2 + \delta}$ for any $\delta > 0$ as $n \to \infty$. For such a time series, one expects such a result to require $T \sim n \log n$ even for a model aware optimal algorithm.

\section{Experiments} \label{sec:experiments_appendix}
We conduct experiments on real-world and synthetic datasets to study the imputation and prediction performance of our algorithm for mixtures of time series processes under varying levels of missing data. Additionally, we present the applicability of our algorithm to the hidden state setting (see Section \ref{sec:hmm}). \\

\noindent {\bf Mixtures of time series processes.} For the synthetically generated datasets, we utilize mixtures of harmonics, trend, and auto-regressive (AR) processes with Gaussian additive noise (since AR is effectively a {\em noisy} version of LRF). When using real-world datasets, we are unaware of the underlying time series processes; nevertheless, these processes appear to display periodicity, trend, and auto-regression. \\

\noindent {\bf Comparisons.} For forecasting, we compare our algorithm to the state-of-the-art time series forecasting library of R, which decomposes a time series into stationary auto-regressive, seasonal, and trend components. 
The library learns each component separately and combines them to produce forecasts. Given that our synthetic and real-world datasets involve additive mixtures of these processes, this serves as a strong baseline to compare against our algorithm. 
We note that we do not outperform optimal model-aware methods for single model classes with all of the data present, at least as implemented in the R-package. However, these methods are not necessarily optimal with missing data and/or when the data is generated by a mixture of multiple model types, which is the setting in which we see our model agnostic method outperform the R-package.
For our imputation experiments, we compare our algorithm against AMELIA II (\cite{amelia2}), which is another R-based package that is widely believed to exhibit excellent imputation performance. \\

\noindent {\bf Metric of evaluation.} Our metric of comparison is the root mean-squared error (RMSE). \\

\noindent {\bf Algorithmic hyper-parameters.} For both imputation and forecasting, we apply the Universal Singular Value Thresholding (USVT) algorithm (\cite{Chatterjee15}) as our matrix estimation subroutine. We use a data-driven approach to choose the singular value threshold $\mu$ and the number of rows in the time series matrix $L$ in our algorithm. Specifically, we reserve 30\% of our training data for cross-validation to pick $\mu$ and $L$. \\

\noindent{\bf Summary of results.} Details of all experiments are provided below. Recall that $p$ is the probability of observation of each datapoint. 

\vspace{.05in}
\noindent {\em Synthetic data:} For forecasting, we determine the forecast RMSE of our algorithm and R's forecast library (see below for how the synthetic data was generated). Our experimental results demonstrate that we outperform R's forecast library, especially under high levels of missing data and noise. For imputation, we outperform the imputation library AMELIA under all levels of missing data.

\vspace{.05in}
\noindent {\em Real-world data:} We test against two real world datasets: (i) Bitcoin price dataset from March 2016 at 30s intervals; (ii) Google flu trends data for Peru from 2003-2012. In both cases, we introduce randomly missing data and then use our algorithm and R's forecast library to forecast into the future. Corroborating the results from the synthetic data experiments, our algorithm's forecast RMSE continues to be lower than that of the R library.

\vspace{.05in}
\noindent {\em Hidden State Model:} We generate a time series according to a Poisson process with latent time-varying parameters. These parameters evolve according to a mixture of time series processes, i.e., sum of harmonics and trends. Our interest is in estimating these time-varying hidden parameters using one realization of integer observations, of which several are randomly missing. For $p$ ranging from $0.3$ to $1.0$, the imputation RMSE is always $< 0.2$ while the $R^2$ is always $> 0.8$, which should be considered excellent. This illustrates the versatility of our algorithm in solving a diverse set of problems. 

\subsection{Synthetically generated data}

We generate a mixture process of harmonics, trend, and auto-regressive components. The first 70\% of the data points are used to learn a model (training) and point-predictions, i.e., forecasts are performed on the remaining 30\% of the data. In order to study the impact of missing data, each entry in the training set is observed independently with probability $p \in (0, 1]$.  \\


\noindent {\bf Forecasts.} Figures \ref{fig:p03}-\ref{fig:p1} visually depict the predictions from our algorithm when compared to the state-of-the-art time series forecasting library in R. We provide the R library the number of lags of the AR component to search over, in effect making its job easier. It is noticeable that the forecasts from the R library always experience higher variance. As $p$ becomes smaller, the R library's forecasts also contain an apparent bias. 
These visual findings are confirmed in Figure \ref{fig:prediction_rmse}, which shows that our algorithm produces a lower RMSE than that of the R forecasting library when working with mixtures of AR, harmonic, and trend processes; in particular, our algorithm's RMSE ranges from $[0.03, 0.11]$ vs. $[0.09, 0.16]$ for R's forecasting library.  \\

\begin{figure}[]
	\centering
	\begin{subfigure}[b]{0.4\textwidth}
		\centering
		\includegraphics[width=\textwidth]{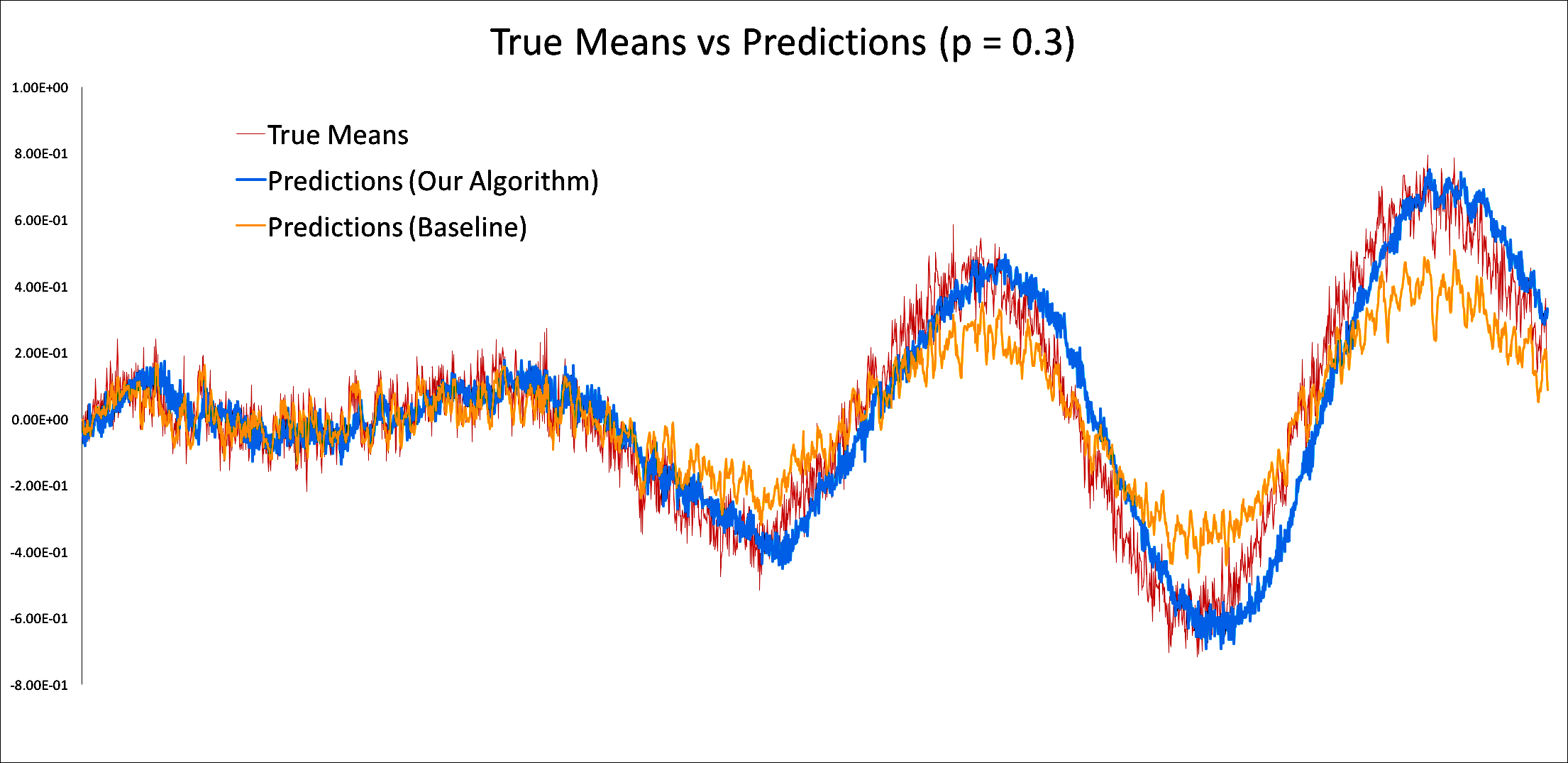}
		\caption{$70\%$ data missing.}
		\label{fig:p03}
	\end{subfigure}
	\begin{subfigure}[b]{0.4\textwidth}
		\centering
		\includegraphics[width=\textwidth]{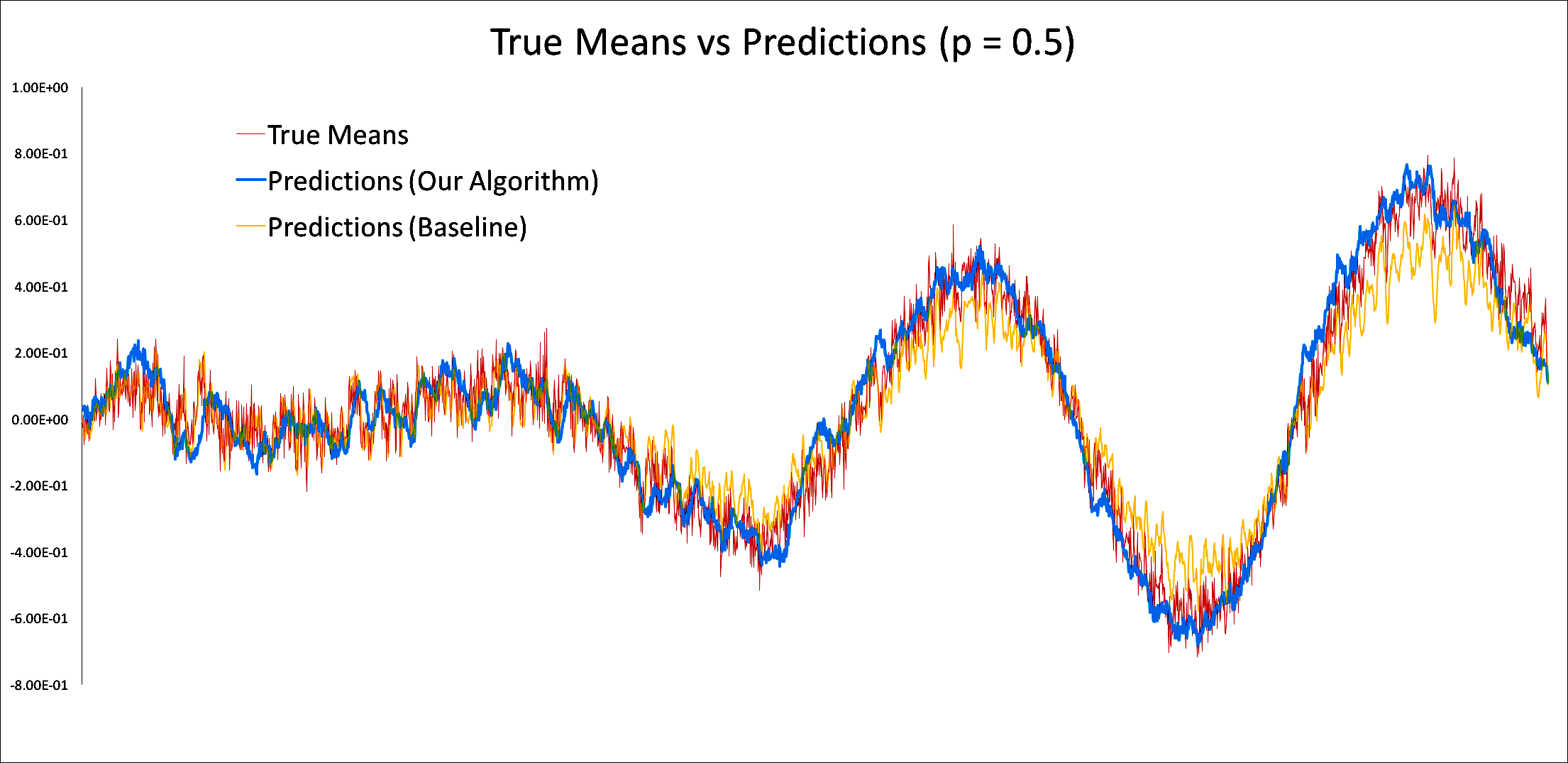}
		\caption{$50\%$ data missing.}
		\label{fig:p05}
	\end{subfigure}
	\begin{subfigure}[b]{0.4\textwidth}
		\centering
		\includegraphics[width=\textwidth]{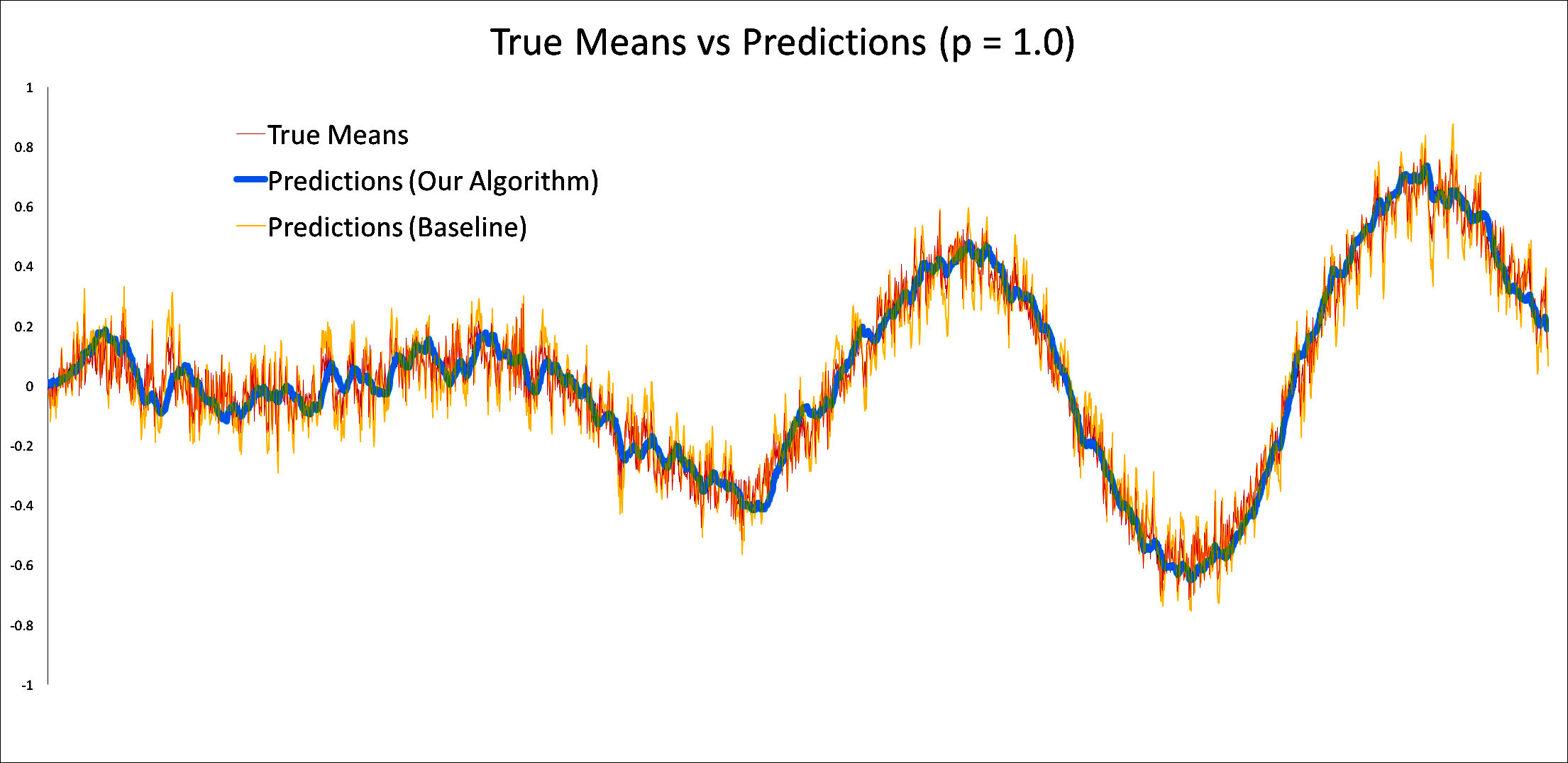}
		\caption{No data missing.}
		\label{fig:p1}
	\end{subfigure}
	\caption{ Plots for three levels of missing data ($p \in \{0.3, 0.5, 1\}$) showing the original time series (means) and forecasts produced by the R-library (baseline) and our algorithm.  }
\end{figure}

\noindent {\bf Imputation.} Figure \ref{fig:imputation_rmse} shows that our algorithm outperforms the state-of-the-art AMELIA library for multiple time series imputation under all levels of missing data. The RMSE of our algorithm ranged between $[0.09, 0.13]$ vs. $[0.14, 0.24]$ for AMELIA. Note that AMELIA is much better than the baseline, i.e., imputing all missing entries with the mean. 

Note that this experiment involved multiple time series where the outcome variable of interest and the $\log$ of its squared power were also included. The additional time series components were included to help AMELIA impute missing values because it is unable to impute missing entries in a single time series. However, our algorithm {\em did not} use these additional time series; instead, our algorithm was only given access to the original time series with missing, noisy observations. 

\begin{figure}[]
	\centering
	\begin{subfigure}[b]{0.4\textwidth}
		\centering
		\includegraphics[width=\textwidth]{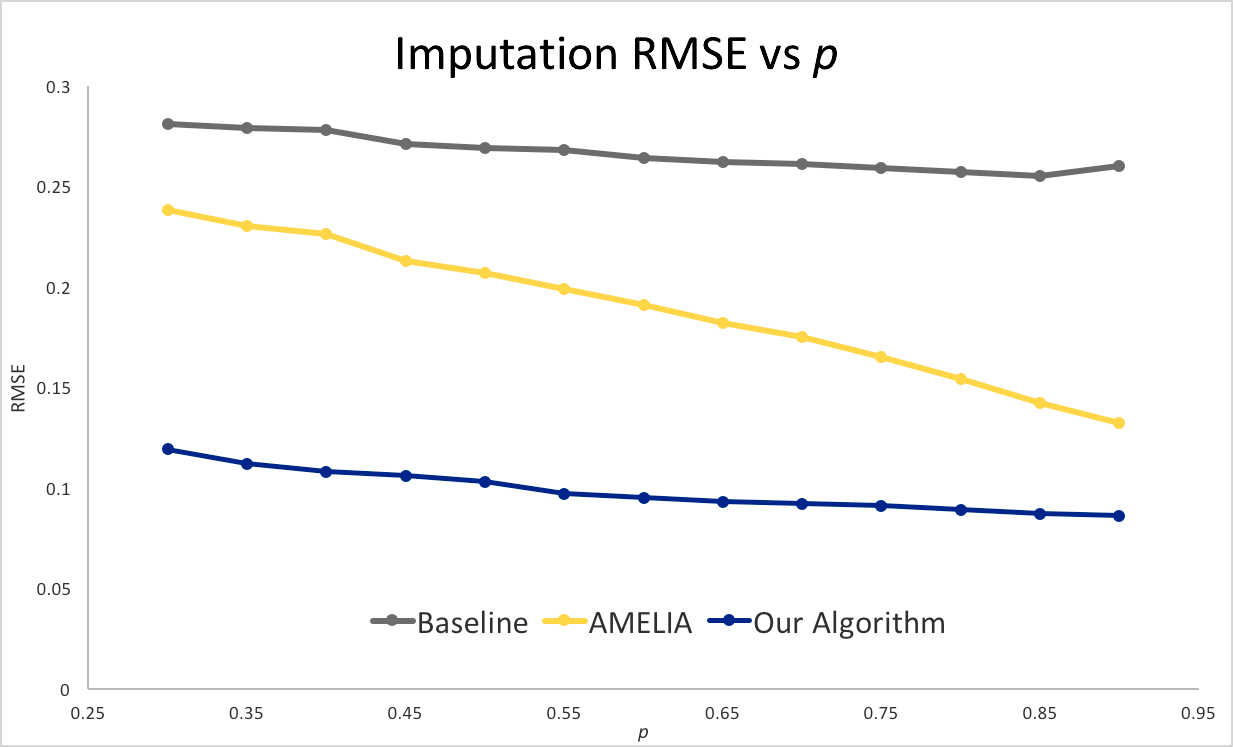}
		\caption{Imputation RMSE (mixture AR, harmonic, trend).}
		\label{fig:imputation_rmse}
	\end{subfigure}\hspace{0.1\textwidth}
	\begin{subfigure}[b]{0.4\textwidth}
		\centering
		\includegraphics[width=\textwidth]{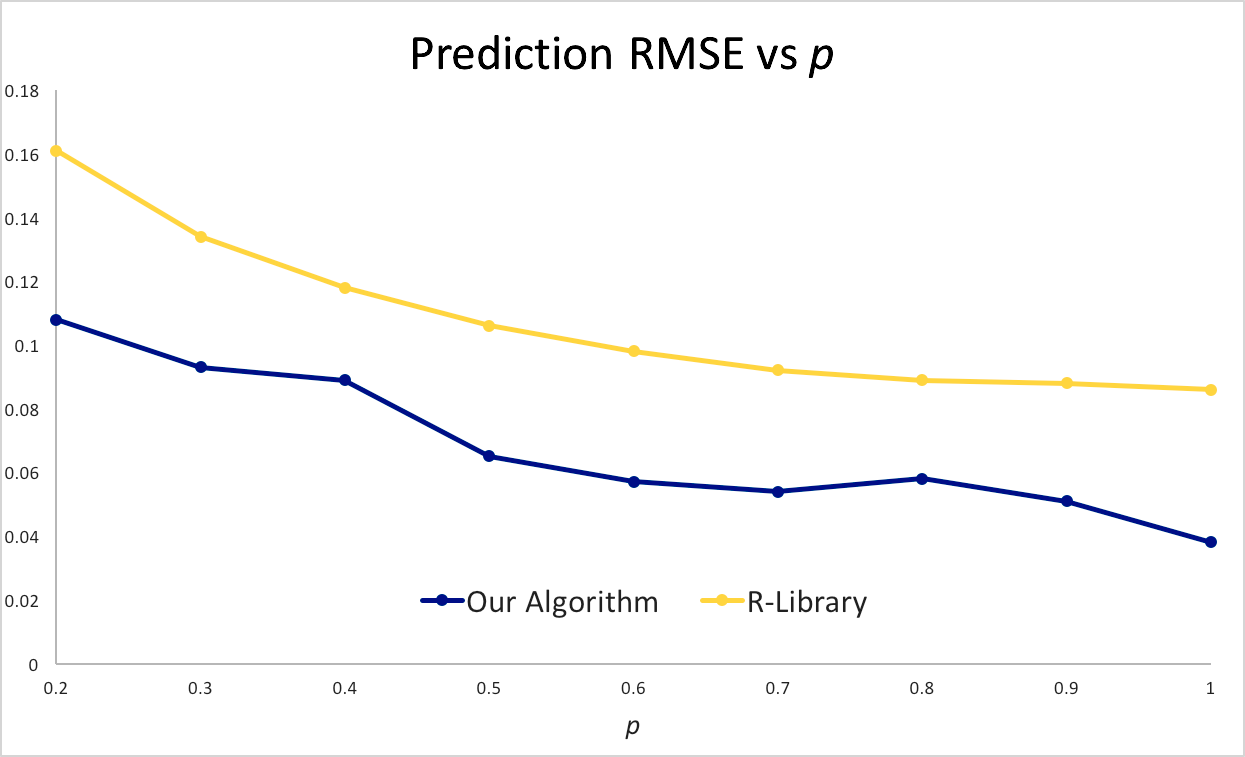}
		\caption{Prediction RMSE (mixture AR, harmonic, trend).}
		\label{fig:prediction_rmse}
	\end{subfigure}
	\caption{Plots showing the Imputation and Prediction RMSE as a function of $p$.}
\end{figure}

\subsection{Real-world data}
We use two real-world datasets to evaluate the performance of our algorithm in situations where the identities of the time series processes are unknown. This set of experiments is intended to highlight the versatility of our algorithm and applicability to practical scenarios involving time series forecasting. We again highlight that for the following datasets, we do not know the true mean processes. Therefore, it is not possible to generate the metric of interest (RMSE) using the means. Instead, we use the observations themselves as the reference to compute the metric. \\

\noindent {\bf Bitcoin.} Figures \ref{fig:bitcoin_pred} and \ref{fig:bitcoin_pred_rmse} show the forecasts for Bitcoin prices (in Yuans) in March 2016 at regular 30s time intervals, which demonstrates classical auto-regressive properties. We provide a week's data to learn and forecast over the next two days. Figure \ref{fig:bitcoin_pred} shows that our algorithm and the R library appear to do an excellent job of predicting the future even with 50\% data missing. Figure \ref{fig:bitcoin_pred_rmse} shows the RMSE of the predictions for our algorithm and the R library as a function of $p$; our algorithm had RMSE's in the range $[0.55, 1.85]$ vs $[0.48, 2.25]$ for the R library, for $p$ ranging from $1.0$ to $0.5$ (note that prices are not normalized). This highlights our algorithm's strength in the presence of missing data. \\

\noindent {\bf Google flu trends (Peru).} Figures \ref{fig:flu_peru_pred} and \ref{fig:flu_peru_pred_rmse} show the forecasts for Google flu search-trends in Peru which shows significant seasonality. We provide weekly data from 2003-2012 to learn and then forecast for each week in the next three years. Figure \ref{fig:flu_peru_pred} shows that our algorithm outperforms R when predicting the future with 30\% data missing. Figure \ref{fig:flu_peru_pred_rmse} shows the RMSE of the predictions as a function of $p$ indicating outperformance of our algorithm under all levels of missing data; our algorithm had RMSE's in range $[8.0, 17.5]$ vs. $[9.0, 26.0]$ for the R library, with $p$ ranging from $1.0$ to $0.5$ (note that prices are not normalized). \\

\begin{figure}[]
	\centering
	\begin{subfigure}[b]{0.4\textwidth}
		\centering
		\includegraphics[width=\textwidth]{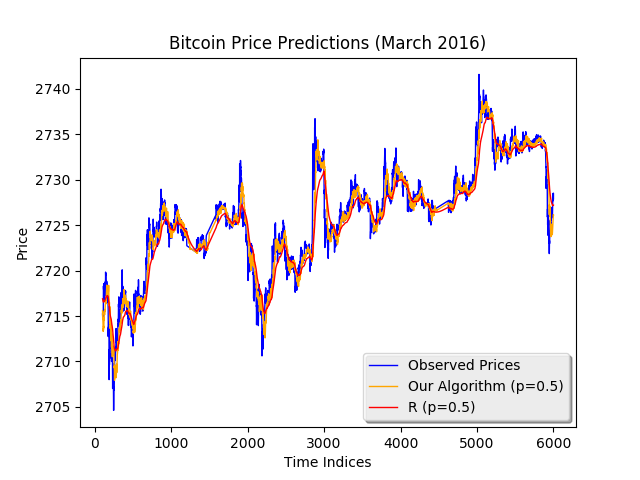}
		\caption{Price predictions for Bitcoin.}
		\label{fig:bitcoin_pred}
	\end{subfigure}\hspace{0.1\textwidth}
	\begin{subfigure}[b]{0.4\textwidth}
		\centering
		\includegraphics[width=\textwidth]{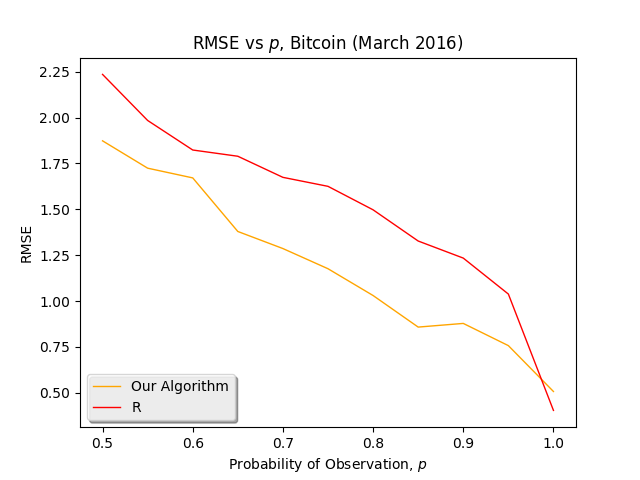}
		\caption{RMSE for Bitcoin predictions.}
		\label{fig:bitcoin_pred_rmse}
	\end{subfigure}
	\caption{ Bitcoin price forecasts and RMSE as a function of $p$.}
\end{figure}

\begin{figure}[]
	\centering
	\begin{subfigure}[b]{0.4\textwidth}
		\centering 
		\includegraphics[width=\textwidth]{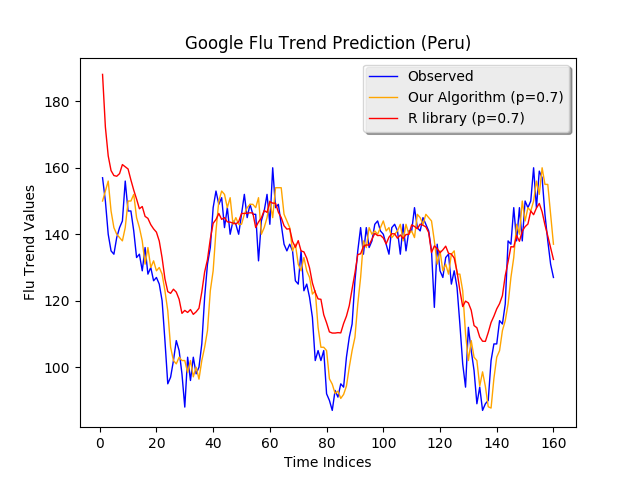}
		\caption{Flu trends predictions (Peru).}
		\label{fig:flu_peru_pred}
	\end{subfigure}\hspace{0.1\textwidth}
	\begin{subfigure}[b]{0.4\textwidth}
		\centering
		\includegraphics[width=\textwidth]{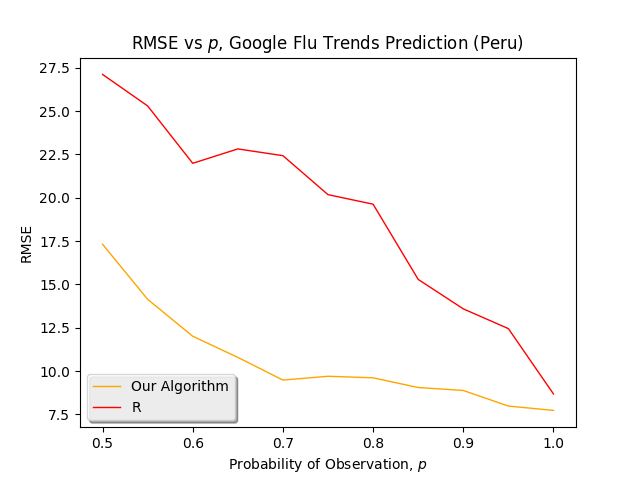}
		\caption{RMSE flu trend predictions.}
		\label{fig:flu_peru_pred_rmse}
	\end{subfigure}
	\caption{Peru's Google flu trends forecasts and RMSE as a function of $p$.}
\end{figure}

\subsection{Hidden state}
We generate a time series from a Poisson process with time-varying parameters, which are hidden. These parameters evolve according to a mixture of sums of harmonics and trends. Our interest is in estimating these time-varying hidden parameters using one realization of integer observations, of which several are randomly missing. Specifically, each point in the original time series is a Poisson random variable with parameter $\lambda(t)$, i.e., $X(t) \sim \text{Poisson}(\lambda(t))$. Further, we let $\lambda(t) = f(t)$, where $f(t)$ is a time-dependent sum of harmonics and logarithmic trend components. Each $X(t)$ is then observed independently with probability $p$ to produce a random variable $Y(t)$. We normalize all parameters and observations to lie between $[-1, 1]$. Observe that $\Ex[Y(t)] = p\lambda(t)$. Note that this is similar to the settings described earlier in this work. It is important to highlight that we have imposed a generic noise model as opposed to an additive noise model. Our goal is to estimate the mean time series process under randomly missing data profiles. 

Figures \ref{fig:hmm_small_99}-\ref{fig:hmm_small_90} show the mean time series process can be estimated via imputation using the algorithm proposed in our work. These two plots show the original time series (with randomly missing data points set to 0), the true means and our estimation. With only 1\% missing data, our algorithm is able to impute the means accurately with the performance degrading slightly with 10\% missing data. We note that these are relatively small datasets with only 25,000 points. Figure \ref{fig:hmm_large_90} shows the same process under 10\% missing data but for 50,000 data points. As expected, our algorithm performs better when given access to a greater number of data points. 

Figure \ref{fig:hmm_metrics} shows plots of RMSE and $R^2$ for the imputed means of the process. Note these apply to the smaller time series of 25,000 data points. The metrics are computed only on the data points that were missing. Observe that the $R^2$ value rises while the RMSE falls as $p$ increases. Both of these profiles confirm our intuition that the imputation improves as a function of $p$. Overall, our performance is fairly robust (RMSE $ < 0.2$ and $R^2 > 0.8$) under all levels of missing data.

\begin{figure}[]
	\centering
	\begin{subfigure}[b]{0.3\textwidth}
		\centering
		\includegraphics[width=\textwidth]{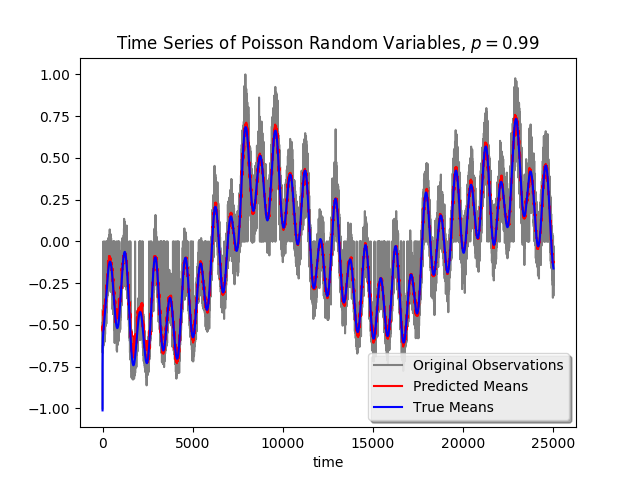}
		\caption{1\% missing data. 25,000 points.}
		\label{fig:hmm_small_99}
	\end{subfigure}\hspace{0.1\textwidth}
	\begin{subfigure}[b]{0.3\textwidth}
		\centering
		\includegraphics[width=\textwidth]{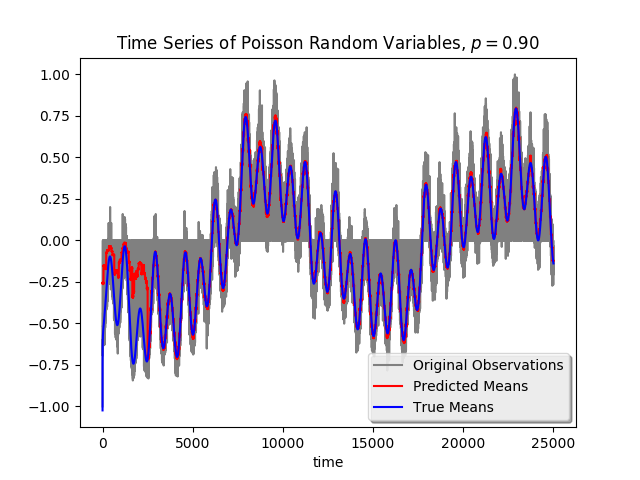}
		\caption{10\% missing data. 25,000 points.}
		\label{fig:hmm_small_90}
	\end{subfigure}
	\begin{subfigure}[b]{0.3\textwidth}
		\centering
		\includegraphics[width=\textwidth]{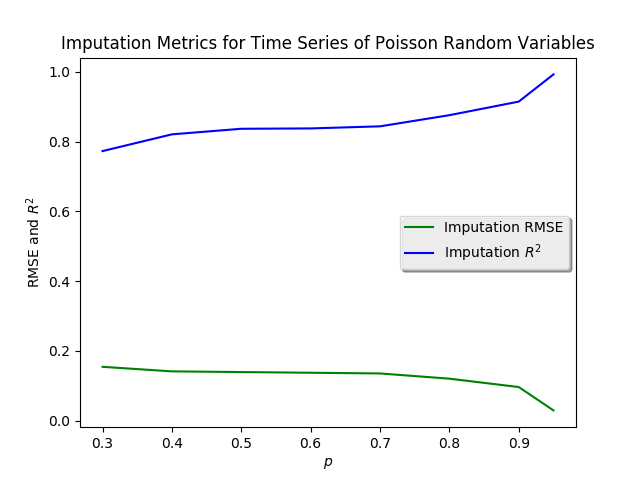}
		\caption{RMSE and $R^2$ vs $p$. 25,000 points.}
		\label{fig:hmm_metrics}
	\end{subfigure}\hspace{0.1\textwidth}
	\begin{subfigure}[b]{0.3\textwidth}
		\centering
		\includegraphics[width=\textwidth]{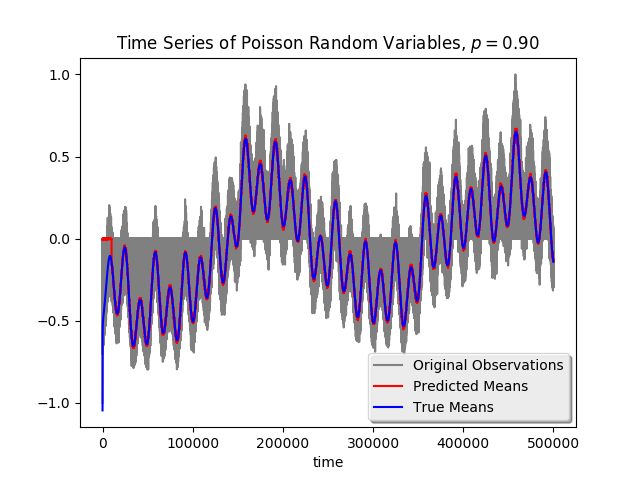}
		\caption{10\% missing data. 50,000 points.}
		\label{fig:hmm_large_90}
	\end{subfigure}
	\caption{Imputation of the means of a Poisson time series. The first three plots correspond to the time series with 25,000 data points and a resulting matrix of dimension $50 \times 500$. The last figure is for the same process, but with twice as much data and matrix dimensions of $100 \times 500$. Note that the randomly missing observations are set to 0 and the entire process is normalized to lie between $[-1, 1]$.}
\end{figure}

\section{Conclusion} \label{sec:conclusion}
In this paper, we introduce a novel algorithm for time series imputation and prediction using matrix estimation methods, which allows us to operate in a model and noise agnostic setting. At the same time, we offer an alternate solution to the error-in-variables regression problem through the lens of matrix estimation. 
We provide finite sample analysis for our algorithm, and identify generic conditions on the time series model class under which our algorithm provides a consistent estimator. As a key contribution, we establish that many popular model classes and their mixtures satisfy these generic conditions. Using synthetic and real-world data, we exhibit the efficacy of our algorithm with respect to a state-of-the-art software implementation available through R. Our experimental results agree with our finite sample analysis. Lastly, we demonstrate that our method can provably recover the hidden state of dynamics, which could be of interest in its own right.

\newpage
\section*{Acknowledgements} \label{sec:acknowledgements}
During this work, the authors were supported in part by a Draper, IDSS WorldQuant, and Thompson Reuters Fellowship, NSF CMMI-1462158, NSF CMMI-1634259 and collaboration with KAIST through MIT project 6937985.  

\bibliographystyle{ACM-Reference-Format}
\bibliography{bibliography}

\newpage
\appendix

\section{Useful Theorems} \label{sec:appendix:useful_theorems}
%

\begin{thm} \label{thm:bernstein} {\bf Bernstein's Inequality. \cite{bernstein1}} \\
	Suppose that $X_1, \dots, X_n$ are independent random variables with zero mean, and M is a constant such that $\abs{X_i} \le M$ with probability one for each $i$. Let $S := \sum_{i=1}^n X_i$ and $v := \text{Var}(S)$. Then for any $t \ge 0$,
	\begin{align*}
		\mathbb{P}(\abs{S} \ge t) &\le 2 \exp(- \dfrac{3 t^2}{6v + 2Mt} ).
	\end{align*} 
\end{thm}


\begin{thm} \label{thm:subgaussian_matrix} {\bf Norm of matrices with sub-gaussian entries. \cite{vershynin2010introduction}} \\
	Let $\bA$ be an $m \times n$ random matrix whose entries $A_{ij}$ are independent, mean zero, sub-gaussian random variables. Then, for any $t > 0$, we have
	\begin{align*}
		\norm{\bA} &\le C K (\sqrt{m} + \sqrt{n} + t)
	\end{align*}
	with probability at least $1 - 2\exp(-t^2)$. Here, $K = \max_{i,j} \norm{A_{ij}}_{\psi_2}$. 
	 \\
\end{thm}

\section{Imputation Analysis}\label{sec:appendix:imputation}

\begin{lemma} \label{lemma:prelims}
Let $\bX$ be an $L \times N$ random matrix (with $L \le N$) whose entries $X_{ij}$ are independent sub-gaussian entries where $\Ex[X_{ij}] = M_{ij}$ and $\norm{X_{ij}}_{\psi_2} \le \sigma$. Let $\bY$ denote the $L \times N$ matrix whose entries $Y_{ij}$ are defined as
\begin{align*}Y_{ij} = \begin{cases}
		X_{ij} & \text{w.p. } p,  \\
		0 & \text{w.p. } 1-p,
		\end{cases}
\end{align*}
for some $p \in (0, 1]$. Let $\hat{p} = \max \Big\{ \frac{1}{LN} \sum_{i=1}^L \sum_{j=1}^N \mathds{1}_{X_{ij} \text{ observed}}, \frac{1}{LN} \Big\}$. Define events $E_1$ and $E_2$ as
\begin{align}
	E_1 &:= \Big\{ \abs{\hat{p} - p} \le p / 20\Big \},
     \\ E_2&:= \Big\{ \norm{\bY - p \bM} \le C_1 \sigma \sqrt{N}\Big \}. 
\end{align}
Then, for some positive constant $c_1$
\begin{align}
	\Pb(E_1) &\ge 1 - 2e^{-c_1LNp} - (1-p)^{LN},
	\\ \Pb(E_2) &\ge 1 - 2 e^{-N}.
\end{align} 
\end{lemma}

\begin{proof}
Let $\hat{p}_0 = \frac{1}{LN} \sum_{i=1}^L \sum_{j=1}^N \mathds{1}_{X_{ij} \text{ observed}}$, which implies $\mathbb{E}[\hat{p}_0] = p$. We define the event
$E_3 := \{ \hat{p}_0 = \hat{p} \}$.
Thus, we have that
\begin{align*}
	\Pb(E_1^c) &= \Pb(E_1^c \cap E_3) + \Pb(E_1^c \cap E_3^c)
	\\ &= \Pb( \abs{ \hat{p}_0 - p} \ge p / 20 ) + \Pb(E_1^c \cap E_3^c)
	\\ &\le \Pb( \abs{ \hat{p}_0 - p} \ge p / 20 )  + \Pb(E_3^c)
	\\ &= \Pb( \abs{ \hat{p}_0 - p} \ge p / 20 )  + (1-p)^{LN},
\end{align*}	
where the final equality follows by the independence of observations assumption and the fact that $\hat{p}_0 \neq \hat{p}$ only if we do not have any observations. 
By Bernstein's Inequality, we have that
\begin{align*}
	\Pb(\abs{ \hat{p}_0 - p} \le p / 20) &\ge 1 - 2e^{-c_1LNp}.
\end{align*}
Furthermore, since $\Ex[Y_{ij}] = pM_{ij}$, Theorem \ref{thm:subgaussian_matrix} yields
\begin{align*}
	\mathbb{P}(E_2) &\ge 1 - 2 e^{-N}.
\end{align*}
\end{proof}

\begin{cor} \label{cor:prelims}
Let $E := E_1 \cap E_2$. Then,
\begin{align}
	\Pb(E^c) &\le C_1 e^{-c_2 N},
\end{align}
where $C_1$ and $c_2$ are positive constants independent of $L$ and $N$. 
\end{cor}

\begin{proof}
By DeMorgan's Law and the Union Bound, we have that
\begin{align} \label{eq:union_bound}
	\Pb(E^c) &= \Pb(E_1^c \cup E_2^c) \nonumber
	\\ &\le \Pb(E_1^c) + \Pb(E_2^c) \nonumber
	\\ &\le C_1 e^{-c_2 N},
\end{align}
where $C_1, c_2 > 0$ are appropriately defined, but are independent of $L$ and $N$. 
\end{proof}

\begin{lemma} \label{lemma:nuclear_norm}
	Let $\bM^{(1)}$ be defined as in Section \ref{sec:notation_and_definitions} and satisfy Property \ref{prop.one}. Then,
	\begin{align*}
		\norm{\bM^{(1)}}_* &\le L \sqrt{N}\delta_1 + \sqrt{rLN} \delta_1 + \sqrt{r} \norm{\bM}_F .
	\end{align*}
\end{lemma}

\begin{proof}
	By the definition of $\bM^{(1)}$ and the triangle inequality property of nuclear norms,
	\begin{align*}
		\norm{\bM^{(1)}}_* &\le \norm{\bM^{(1)} - \bM_{(r)}}_* + \norm{\bM_{(r)}}_*
		\\ &\stackrel{(a)}{\le} \sqrt{L} \norm{\bM^{(1)} - \bM_{(r)}}_F  + \norm{\bM_{(r)}}_*
		\\ &\stackrel{(b)}{\le} L\sqrt{N}\delta_1 + \norm{\bM_{(r)}}_*.
	\end{align*}
	Note that (a) makes use of the fact that $\norm{\bQ}_* \le \sqrt{\text{rank}(\bQ)} \norm{\bQ}_F$ for any real-valued matrix $\bQ$ and (b) utilizes Property \ref{prop.one}. Since $\text{rank}(\bM_{(r)}) = r$,  we have $ \norm{\bM_{(r)}}_* \le  \sqrt{r} \, \norm{\bM_{(r)}}_F$. Applying triangle inequality and Property \ref{prop.one} again further yields
	\begin{align*}
		\norm{\bM_{(r)}}_F &\le \norm{\bM_{(r)} - \bM}_F + \norm{\bM}_F
		\le \sqrt{LN} \delta_1 + \norm{\bM}_F. 
	\end{align*}
	This completes the proof. 
\end{proof}

\begin{thm*} [\ref{thm:imputation}] 
Assume Property \ref{prop.one} holds and $\text{ME}$ satisfies Property \ref{prop:2.1}. Then for some $C_1, C_2, C_3, c_4 > 0$, 
\begin{align*}
	\emph{MSE}(\hat{f}_I, f) \le \frac{C_1\sigma}{p}\bigg(\frac{L N \delta_1}{\norm{f}_2^2} + \frac{\sqrt{rL} N \delta_1}{\norm{f}_2^2} + \frac{\sqrt{r N}}{\norm{f}_2}\bigg) + \dfrac{C_2(1-p)}{pLN} + C_{3} e^{-c_{4} N}.
\end{align*}
\end{thm*}

\begin{proof}
By \eqref{eq:imputation_matrix}, it suffices to analyze the time series imputation error by measuring the relative mean-squared error of $\bhM^{(1)}$. For notational simplicity, let us drop the superscripts on $\bhM^{(1)}$ and $\bM^{(1)}$. Let $E := E_1 \cap E_2$, where $E_1$ and $E_2$ are defined as in Lemma \ref{lemma:prelims}. By the law of total probability, we have that
\begin{align} \label{eq:imputation1}
	\Ex \norm{\bhM - \bM}_F^2 &\le \Ex \Big[ \norm{\bhM - \bM}_F^2 \mid E \Big] + \Ex \Big[ \norm{\bhM - \bM}_F^2 \mid E^c \Big] \Pb (E^c). 
\end{align}
We begin by bounding the first term on the right-hand side of \eqref{eq:imputation1}. By Property \ref{prop:2.1} and assuming $E$ occurs, we have that
\begin{align*}
	\norm{\hat{p} \bhM - p \bM}_F^2 &\le C_1 \norm{\bY - p\bM} \, \norm{p \bM}_*
	\le C_2 \sigma \sqrt{N} \, \norm{\bM}_*.
\end{align*}
Therefore, 
\begin{align*}
	p^2 \norm{\bhM - \bM}_F^2 &\le C_3 \hat{p}^2 \norm{ \bhM - \bM}_F^2
	\\ &\le C_3  \norm{ \hat{p} \bhM - p\bM}_F^2 +   C_3 (\hat{p} - p)^2 \norm{\bM}_F^2
	\\ &\le C_4 p \sigma \sqrt{N} \, \norm{ \bM}_* + C_3 (\hat{p} - p)^2 \norm{f}_2^2
\end{align*}
for an appropriately defined $C_4$. Observe that $\Ex(\hat{p} - p)^2 = p(1-p)/LN$. 
Thus using Corollary \ref{cor:prelims} and taking expectations, we obtain
\begin{align*}
	\Ex \norm{\bhM - \bM}_F^2 &\le C_4 p^{-1} \sigma \sqrt{N} \, \norm{\bM}_* + \dfrac{C_3 (1-p) \norm{f}_2^2}{p LN} +  C_5 \norm{f}_2^2 e^{-c_6 N}. 
\end{align*}
Normalizing by $\norm{f}_2^2$ gives
\begin{align*}	
	\text{MSE}(\hat{f}_I, f) &\le \dfrac{C_4 \sigma \sqrt{N} \, \norm{\bM}_* }{p \, \norm{f}_2^2} + \dfrac{C_3(1-p)}{pLN} + C_{5} e^{-c_{6} N}.
\end{align*}
Invoking Lemma \ref{lemma:nuclear_norm}, we obtain
\begin{align*}
	\text{MSE}(\hat{f}_I, f) &\le  \frac{C_4\sigma}{p}\bigg(\frac{L N \delta_1}{\norm{f}_2^2} + \frac{\sqrt{rL} N \delta_1}{\norm{f}_2^2} + \frac{\sqrt{r N}}{\norm{f}_2}\bigg) + \dfrac{C_3(1-p)}{pLN} + C_{5} e^{-c_{6} N}.
\end{align*}
The proof is complete after relabeling constants. 

\end{proof}

\section{Forecast Analysis}\label{sec:appendix:forecast}

Let us begin by analyzing the forecasting error for any $k \in [L]$. 

\begin{lemma} \label{lemma:forecasting_new}
For each $k \in [L]$, assume Property \ref{prop.two} holds and $\text{ME}(\cdot)$ satisfies Property \ref{prop:2.2}. Then,
	\begin{align}
		\Ex \Bigg[ \sum_{t \in S_k} \Big(\hat{f}_F(t) - f(t) \Big)^2 \Bigg] &\le   \Big( \delta_2 + \sqrt{C_\beta N} \delta_{3} \Big)^2 + 2 \sigma^2 \hat{r}_k.
	\end{align}
	Here, $S_k := \{t \in [T]: (t \mod L) + 1 = k\}$ and $\hat{r}_k := \text{rank}(\bhtkM)$.
\end{lemma}

\begin{proof}
Observe that we can write
	\begin{align}
		\Ex \norm{M_L^{(k)} - (\bhtkM)^T \hat{\beta}^{(k)}}_2^2 &\equiv \Ex \Bigg[ \sum_{t \in S_k} \Big(\hat{f}_F(t) - f(t) \Big)^2 \Bigg].
	\end{align}
	For notational simplicity, let $\bQ := (\btM^{(k)})^T$ and $\bhQ := (\bhtkM)^T$. Similarly, we will drop all superscripts $(k)$ throughout this analysis for notational ease. Recall $X_L = M_L + \epsilon_L$. Then note that by the definition of the optimization in step 2 of the forecast algorithm,
	\begin{align} \label{eq:linear_1}
		\norm{X_L -  \bhQ \hat{\beta}}_2^2  &\le \norm{X_L -  \bhQ \beta^*}_2^2  \nonumber
		\\ &= \norm{M_L -  \bhQ \beta^*}_2^2 +  \norm{\epsilon_L}_2^2 + 2 \epsilon_L^T (M_L - \bhQ \beta^*).
	\end{align}
	Moreover,
	\begin{align} \label{eq:linear_2}
		\norm{X_L -  \bhQ \hat{\beta}}_2^2  &= \norm{M_L - \bhQ \hat{\beta}}_2^2  + \norm{\epsilon_L}_2^2 - 2 \epsilon_L^T (\bhQ \hat{\beta} - M_L). 
	\end{align}
	Combining \eqref{eq:linear_1} and \eqref{eq:linear_2} and taking expectations, we have
	\begin{align} \label{eq:linear_3}
		\Ex \norm{ M_L - \bhQ \hat{\beta}}_2^2 &\le \Ex \norm{M_L - \bhQ \beta^*}_2^2  + 2 \Ex[\epsilon_L^T \bhQ (\hat{\beta} - \beta^*)].
	\end{align}
	Let us bound the final term on the right hand side of \eqref{eq:linear_3}. Under our independence assumptions, observe that
	\begin{align}
		\Ex[\epsilon_L^T \bhQ] \beta^* &= \Ex[\epsilon_L^T] \Ex[\bhQ] \beta^* = 0. 
	\end{align}
	Recall $\hat{\beta} = \bhQ^{\dagger} X_L = \bhQ^{\dagger} M_L + \bhQ^{\dagger} \epsilon_L$. Using the cyclic and linearity properties of the trace operator (coupled with similar independence arguments), we further have
	\begin{align} \label{eq:linear_trace}
		\Ex [\epsilon_L^T \bhQ \hat{\beta}] &= \Ex[\epsilon_L^T \bhQ \bhQ^{\dagger}] M_L + \Ex[\epsilon_L^T \bhQ \bhQ^{\dagger} \epsilon_L] \nonumber
		\\ &= \Ex \Big[ \text{Tr}\Big( \epsilon_L^T  \bhQ \bhQ^{\dagger} \epsilon_L \Big) \Big] \nonumber
		\\ &= \Ex \Big[ \text{Tr}\Big( \bhQ \bhQ^{\dagger} \epsilon_L \epsilon_L^T \Big) \Big] \nonumber
		\\ &= \text{Tr}\Big( \Ex [ \bhQ \bhQ^{\dagger} ] \cdot \Ex [ \epsilon_L \epsilon_L^T ]  \Big) \nonumber
 		\\ &\le \sigma^2 \Ex \Big[ \text{Tr} \Big( \bhQ \bhQ^{\dagger} \Big) \Big].
	\end{align}
	Let $\bhQ = \bU \bS \bV^T$ be the singular value decomposition of $\bhQ$. Then
	\begin{align} \label{eq:linear_svd}
		 \bhQ \bhQ^{\dagger} &= \bU \bS \bV^T \bV \bS^{\dagger} \bU^T \nonumber
		 \\ &= \bU \tilde{\bI} \bU^T.
	\end{align}
	Here, $\tilde{\bI}$ is a block diagonal matrix where its nonzero entries on the diagonal take the value 1. Plugging in \eqref{eq:linear_svd} into \eqref{eq:linear_trace}, and using the fact that the trace of a square matrix is equal to the sum of its eigenvalues, 
	\begin{align} 
		\sigma^2 \Ex \Big[ \text{Tr} \Big( \bhQ \bhQ^{\dagger} \Big) \Big] &= \sigma^2 \Ex [\text{rank}(\bhQ)]. 
	\end{align}
	We now turn our attention to the first term on the right hand side of \eqref{eq:linear_3}. By Property \ref{prop.two}, we obtain
	\begin{align*}	
		\norm{M_L - \bhQ \beta^*}_2 &= \norm{M_L - (\bQ - \bQ + \bhQ) \beta^*}_2
		\\ &\le \norm{M_L - \bQ \beta^*}_2 + \norm{(\bQ - \bhQ) \beta^*}_2 
		\\ &\le \delta_2 + \norm{(\bQ - \bhQ) \beta^*}_2.
	\end{align*}
	Thus we have that
	\begin{align} \label{eq:annoying}
		\Ex \norm{(\bQ - \bhQ) \beta^*}_2 &= \Ex \norm{(\btM - \bhtM)^T \beta^*}_2 
		\\ &\le \sum_{i = 1}^{L-1}  \abs{\beta^*_i} \cdot \Ex \Bigg[ \Big(\sum_{j = 1}^N  ( \hat{M}_{ij} - M_{ij} )^2\Big)^{1/2} \Bigg]
		\\ &\le \norm{\beta^*}_1 \cdot  \Ex \Bigg[ \Big(\max\limits_{1 \le i < L} \sum_{j = 1}^N  ( \hat{M}_{ij} - M_{ij} )^2\Big)^{1/2} \Bigg]
		\\ & =:  C_\beta \sqrt{N} \cdot \text{MRSE}(\bhtM, \btM). 
	\end{align} 
	Putting everything together, we obtain our desired result. 
	\end{proof}
	
	\medskip
	

\medskip


\begin{thm*} [\ref{thm:asymptotics}]
Assume Property \ref{prop.two} holds and $\text{ME}$ satisfies Property \ref{prop:2.2}, with $p \ge p^*(L, N)$. Let $\hat{r} := \max\limits_{k \in [L]} \emph{rank}(\bhtkM)$. Then, 
\[
\emph{MSE}(\hat{f}_F, f) \le \frac{1}{N-1} \Big( ( \delta_2 + \sqrt{C_\beta N}  \delta_3 )^2 + 2 \sigma^2 \hat{r} \Big).
\]
\end{thm*}

\begin{proof}
	For simplicity, define $\delta(k) := ( \delta_2 + \sqrt{N} \delta_{3})^2 + 2 \sigma^2 \hat{r}_k$. By Lemma \ref{lemma:forecasting_new}, for all $k \in [L]$ we have
	\begin{align}
		\Ex \Bigg[ \sum_{t \in S_k} \Big(\hat{f}_F(t) - f(t) \Big)^2 \Bigg] &\le \delta(k). 
	\end{align}
	Let $\delta_{\max} :=  ( \delta_2 +  \sqrt{C_\beta N} \delta_3)^2 + 2 \sigma^2 \hat{r}$. Recall $S_k := \{t \in [T]: (t \mod L) + 1 = k\}$. Then, it follows that 
		\begin{align*}
		\text{MSE}(\hat{f}_F, f) &\le  \frac{\delta_{\max}}{N-1} . 
	\end{align*}
\end{proof}

\section{Model Analysis} \label{sec:appendix:models}
We first define a somewhat technical Property \ref{prop.three}, that will aid us in proving that the various models in Section \ref{sec:model} satisfy Property \ref{prop.one} and \ref{prop.two}. Recall $f$ is the underlying time series we would like to estimate. Define $\eta_k: \mathbb{Z} \times \mathbb{Z} \to \mathbb{R}$ such that
\begin{align}\label{eqref:eta_definition}
\eta_k(\theta_i, \rho_j) \coloneqq f(i + (j-1)L + (k-1)),
\end{align}
where $\theta_i = i$ and $\rho_j = (j-1)L + (k-1)$.

Intuitively, \eqref{eqref:eta_definition} is representing $f(t)$ as a function of two parameters: $\theta_i = i$ and $\rho_j = (j-1)L + (k-1)$. As a result, we can express $f$ as a latent variable model, a representation which is very amenable to theoretical analysis in the matrix estimation literature. Specifically, $[M^{(k)}_{ij}] = [\eta_k(\theta_i, \rho_j)]$ by the construction of $\bM^{(k)}$. Effectively, the latent parameters $(\theta_i, \rho_j)$ encode the amount of shift in the argument to $f(t)$ so as to obtain the appropriate entry in the matrix $\bM^{(k)}$.

\begin{property}\label{prop.three}
For all $k \in [L]$, let matrices $\bX^{(k)}$ and $\bM^{(k)}$ satisfy the following: 
	\begin{itemize}
		\item[]{\bf A.} For each $i \in [L]$ and $j \in [N]$: 
		\begin{itemize}
			\item[1.] $X_{ij}^{(k)}$ are independent sub-gaussian random variables with $\Ex[X_{ij}^{(k)}] = M_{ij}^{(k)}$ and $\norm{X_{ij}^{(k)}}_{\psi_2} \le \sigma$.
			\item[2.] $X^{(k)}_{ij}$ is observed with probability $p \in (0, 1]$, independently. 
		\end{itemize}
		\item[]{\bf B.}  There exists $\bM_{(r)} \in \mathbb{R}^{L \times N}$ such that: 
		\begin{itemize}
			\item[1.] $\bM_{(r)}$ has $r_4$ distinct rows where $r_4 < L$.
			\item[2.] $\norm{\bM^{(k)} - \bM_{(r)}}_{\max} \le \delta_4$.
		\end{itemize}
	\end{itemize}
\end{property}
\noindent We begin with Proposition \ref{prop:lastrow}, which motivates the use of linear methods in forecasting. 

\begin{prop}\label{prop:lastrow}
For all $k \in [L]$, let $\bM^{(k)}$, defined as in Section \ref{sec:notation_and_definitions}, satisfy Property \ref{prop.three}. Then, there exists a $\beta^*$ such that 
	\begin{align*}
		\norm{M^{(k)}_L - (\btM^{(k)})^T\beta^*}_2 &\le 2 \delta_4 \sqrt{N},
	\end{align*}
where $\norm{\beta^*}_0 = 1$.
\end{prop}
	
\begin{proof} We drop the dependence on k from $\bM^{(k)}$ and $\eta_k$ for notational convenience. Furthermore, we prove it for the case of $k=1$ since the proofs for a general $k$ follow from identical arguments after first making an appropriate shift in the entries of the matrix of interest.
Assume we have access to data from $X[1 \colon T + r_4 - 1]$. 
Let us first construct a matrix with overlapping entries, $\blineM= [\lineM_{ij}] = [f(i + j-1)]$, of dimension $L \times (T + r_4 - 1)$. 
We have $\lineM_{ij}= \eta(\bar{\theta}_i, \bar{\rho}_j)$ with $\bar{\theta_i} = i$ and $\bar{\rho}_j = (j-1)$, where $\eta$ is as defined in \eqref{eqref:eta_definition}.
By construction, the skew-diagonal entries from left to right of $\blineM$ are constant, i.e.,
\begin{align} \label{eq:skew_diagonal}
	\lineM_{ki}&:= \{ \lineM_{k-j, i+j}: 1 \le k - j \le L, 1 \le i + j \le T + r_4 - 1\}. 
\end{align}
Under this setting, we note that the columns of $\bM$ are subsets of the columns of $\blineM$. Specifically, for all $0 \le j < N$ and $k \le L$, 
\begin{align} \label{eq:column_subset}
	\lineM_{k, jL +1}&= M_{k, j+1}. 
\end{align}
Analogously to how $\blineM$ was constructed with respect to $\bM$, we define $\blineM_{(r)}$ with respect to $\bM_{(r)}$.

Observe that by construction, every entry within $\blineM$ exists within $\bM$. 
Hence, $\lineM_{i, j}= M_{i', j'}, \ \lineM^{(r)}_{i, j}= M^{(r)}_{i', j'}$ for some $(i', j')$, and
\begin{align*}
	\abs{\lineM_{i,j} - \lineM^{(r)}_{i, j}} &= 
	\abs{ M_{i,j} - M^{(r)}_{i, j} }
	\\ &\le \norm{\bM- \bM_{(r)}}_{\max} 
	\\ &\le \delta_4,
\end{align*}
where the inequality follows from Condition B.2 of Property \ref{prop.three}. 

By Condition B.1 of Property \ref{prop.three} and applying the Pigeonhole Principle, we observe that within the last $r_4+1$ rows of $\bM_{(r)}$, at least two rows are identical. Without loss of generality, let these two rows be denoted as $M_{L-r_1}^{(r)} = [M_{L-r_1, i}^{(r)}]_{i \le N}$ and $M_{L-r_2}^{(r)} = [M_{L-r_2, i}^{(r)}]_{i \le N}$, respectively, where $r_1 \in \{1, \dots, r_4-1\}$, $r_2 \in \{2, \dots, r_4\}$, and $r_1 < r_2$. 
Consequently, it must be the case that the same two rows in $\blineM_{(r)}$ are also identical; i.e., for all $i \le T + r_4 - 1$, 
\begin{align} \label{eq:pigeon1}
	\lineM_{L-r_1, i}^{(r)} &= \lineM_{L-r_2, i}^{(r)}. 
\end{align} 
Using this fact, we have that for all $i \le T + r_4 - 1$,
\begin{align} \label{eq:pigeon2}
	\abs{ \lineM_{L-r_1, i}- \lineM_{L-r_2, i}} &\le \abs{\lineM_{L-r_1, i}- \lineM_{L-r_1, i}^{(r)}} + \abs{\lineM_{L-r_2, i}- \lineM_{L-r_2, i}^{(r)}} 
	+ \abs{\lineM_{L-r_1, i}^{(r)} - \lineM_{L-r_1, i}^{(r)}} 
	\le 2 \delta_4,
\end{align}
where the last inequality follows from \eqref{eq:pigeon1} and the construction of $\blineM_{(r)}$. Additionally, by the skew-diagonal property of $\blineM$ as described above by \eqref{eq:skew_diagonal}, we necessarily have the following two equalities: 
\begin{align}
	\lineM_{Li}&= \lineM_{L-r_1, r_1 + i }\label{eq:r1}
	\\ \lineM_{L - \Delta_r, i}&= \lineM_{L-r_2, r_1 + i}\label{eq:r2},
\end{align}
where $\Delta_r = r_2 - r_1$. Thus, by \eqref{eq:pigeon2}, \eqref{eq:r1}, and \eqref{eq:r2}, we obtain for all $i \le T$,
\begin{align} \label{eq:L_delta}
	\abs{ \lineM_{Li}- \lineM_{L - \Delta_r, i}} &= \abs{\lineM_{L-r_1, r_1 + i }- \lineM_{L-r_2, r_1 + i}} \nonumber
	\\ &\le 2 \delta_4. 
\end{align}
Thus, applying \eqref{eq:column_subset} and \eqref{eq:L_delta}, we reach our desired result, i.e., for all $i \le N$,
\begin{align}
	\abs{ M_{Li}- M_{L - \Delta_r, i}} &\le 2 \delta_4.
\end{align}
Recall $\btM = [M_{ij}]_{i < L, j \le N}$ excludes the last row of $\bM$. From above, we know that there exists some row $\ell := L-\Delta_r < L$ such that $\norm{M_L- M_{\ell}}_2 \le 2 \delta_4 \sqrt{N}$. Clearly, we can express
\begin{align} \label{eq:beta}
	M_{\ell}&= \btM^T \beta^*,
\end{align}
where $\beta^* \in \mathbb{R}^{L-1}$ is a 1-sparse vector with a single nonzero component of value 1 in the $\ell$th index. This completes the proof.

\end{proof}

\begin{cor}\label{corollary:property3_property1and2}
For all $k \in [L]$, let $\bM^{(k)}$, defined as in Section \ref{sec:notation_and_definitions}, satisfy Property \ref{prop.three} with $\delta_4, r_4$. Then $\bM^{(k)}$ obeys,
\begin{itemize}
\item[(i)] Under {\em Model Type 1}, Property \ref{prop.one} is satisfied with $\delta_1 = \delta_4$ and $r = r_4$.
\item[(ii)] Under {\em Model Type 2}, Property \ref{prop.two} is satisfied with $\delta_2 = 2 \delta_4 \sqrt{N}$.
\end{itemize}
\end{cor}

\begin{proof}
Condition A of both Property \ref{prop.one} and \ref{prop.two} is satisfied by definition. (i) Condition B.1, B.2 of Property \ref{prop.three} together imply Condition B of Property \ref{prop.one} for the same $\delta_1, r_4$. (ii) Proposition \ref{prop:lastrow} implies Condition B of Property \ref{prop.two} by scaling $\delta_4$ with $2\sqrt{N}$.
\end{proof}

\subsection{Proof of Proposition \ref{prop:lowrank_LTI}}
\begin{prop*}[\ref{prop:lowrank_LTI}] {\color{white} .}  
\begin{itemize}
	\item[(i)] Under {\em Model Type 1}, $f^{\emph{LRF}}$ satisfies Property \ref{prop.one} with $\delta_1 = 0$ and $r = G$; 
	\item[(ii)] Under {\em Model Type 2}, $f^{\emph{LRF}}$ satisfies Property \ref{prop.two} with $\delta_2 = 0$ and $C_\beta = C \cdot G$ where $C> 0$ is an absolute constant.
\end{itemize}
\end{prop*}

\begin{proof}
Let $f(t) = f^{\text{LRF}}$. By definition of $f(t)$, we have that for all $i \in \{G+1, \dots, L\}$ and $j \in \{1, \dots N\}$,
\begin{align*}
M^{(k)}_{ij} &= f(i + (j-1)L + (k - 1)) \\
&= \sum_{g=1}^G \alpha_g f((i -g) + (j-1)L + (k - 1)) \\
&= \sum_{g=1}^G \alpha_g M^{(k)}_{(i-g)j}.
\end{align*}
In particular, $M^{(k)}_{Lj} = \sum_{g=1}^G \alpha_g M^{(k)}_{(L-g)j}$ for all $j \in \{1, \dots N\}$, and so we immediately have condition (ii) of the Proposition with $C = \max_{g \in G} \alpha_g$. Since every row from $G+1, \dots, L$ is a linear combination of the rows above, the rank of $\bM^{(k)}$ is at most $G$. Ergo, we have condition (i) of the Proposition.
\end{proof} 

\begin{prop} \label{prop:lrf_decomposition}
	Let $f(t) = f^{\text{LRF}}$ be defined as in \eqref{prop:lowrank_LTI}. Then, for any given $L \geq 1$ and $N \geq 1$, for all $1\leq s \leq L, ~1\leq t \leq N$, 
	$f$ admits decomposition
	\begin{align}
		f(t+s) &= \sum_{g=1}^G \alpha_g a_g(t) b_g(s)
	\end{align}
	for some scalars $\alpha_g$ and functions $a_g: [L] \to \mathbb{R}$, $b_g: [N] \to \mathbb{R}$. 
\end{prop}

\begin{proof}
Let $T = LN$, consider $f$ restricted to $\{1,\dots, T = L N\}$. Now, by Proposition \ref{prop:lowrank_LTI}, we have that the rank of $\bM^{(k)}$ is at most $G$. Thus, the singular value decomposition of $\bM^{(k)}$ has the form
	\begin{align*}
		\bM^{(k)} &= \sum_{g=1}^G \alpha_{g} a_{g} b_{g}^T,
	\end{align*}
	where $\alpha_{g}$ are the singular values, and $a_{g}, b_{g}$ are the corresponding left and right singular vectors of $\bM^{(k)}$, respectively. Therefore, the $(i,j)$-th entry of $\bM^{(k)}$ has the form
	\begin{align} \label{eq:lrf_form}
		M_{ij}^{(k)} &= f(i + (j-1)L + (k-1)) =\sum_{g=1}^G \alpha_g a_g(i) b_g(j),
	\end{align}
	where $a_g(i)$ corresponds to the $i$-th entry of the $g$-th left singular vector, and $b_g(j)$ corresponds to the $j$-th entry of the $g$-th right singular vector. Thus, $a_g: [L] \to \mathbb{R}$ and $b_g: [N] \to \mathbb{R}$. 
\end{proof}


\begin{cor*} [\ref{corollary:lowrank_LTI_imputation}]
Under {\em Model Type 1}, let the conditions of Theorem \ref{thm:imputation} hold. Let $N = L^{1 + \delta}$ for any $\ \delta > 0$. Then for some $C > 0$, if 
\[ 
T \ge C \Bigg(\frac{G}{\delta_{\text{error}}^2}\Bigg)^{2 + \delta},
\]
we have $\emph{MSE}(\hat{f}_I, f^{\emph{LRF}}) \le \delta_{\text{error}}$.
\end{cor*}

\begin{proof}
By Proposition \ref{prop:lowrank_LTI}, we have for some $C_1, C_2, C_3, c_4 > 0$  
\[
\text{MSE}(\hat{f}_I, f^{\emph{LRF}}) \le \frac{C_1 \sigma}{p}\sqrt{\frac{G}{L}} + C_2 \dfrac{(1-p)}{LNp} + C_{3} e^{-c_{4} N}.
\]
We require the r.h.s of the term above to be less than $\delta_{\text{error}}$. Thus, we have that
\begin{align*}
 \frac{C_1 \sigma}{p}\sqrt{\frac{G}{L}} + C_2 \dfrac{(1-p)}{LNp} + C_{3} e^{-c_{4} N} &\stackrel{(a)} \le C \Bigg( \sqrt{\frac{G}{L}} + \frac{1}{LN} + e^{-c_{4} N} \Bigg) \\
&\stackrel{(b)} \le C \Bigg( \sqrt{\frac{G}{L}} \Bigg) 
\end{align*}
where (a) follows for appropriately defined $C > 0$ and by absorbing $p, \sigma$ into the constant; (b) follows since $\frac{1}{LN} \le \frac{G}{L}$ and $e^{-c_{4} N} \le \sqrt{\frac{G}{L}}$ for sufficiently large $L, N$ and by redefining $C$. Hence, it suffices that $\delta_{\text{error}} \ge C \Bigg( \sqrt{\frac{G}{L}} \Bigg) \implies \ T \ge C \Bigg(\frac{G}{\delta_{\text{error}}^2}\Bigg)^{2 + \delta}$.
\end{proof} 

\begin{cor*} [\ref{corollary:lowrank_LTI_forecasting}]
Under {\em Model Type 2}, let the conditions of Theorem \ref{thm:asymptotics} hold. Let $N = L^{1 + \delta}$ for any $\delta > 0$. Then for some $C > 0$, if 
\[ 
T \ge C \Bigg( \frac{\sigma^2}{\delta_{\text{error}} - G\delta_3^2} \Bigg)^{\frac{2+\delta}{\delta}}
\]
we have $\emph{MSE}(\hat{f}_F, f^{\emph{LRF}})  \le \delta_{\text{error}}$.
\end{cor*}

\begin{proof}
By Proposition \ref{prop:lowrank_LTI}, we have
\[
\emph{MSE}(\hat{f}_F, f^{\emph{LRF}})  \le \frac{1}{N-1} (G \delta_3^2 \, N + 2 \sigma^2 \hat{r} ).
\]
We require the r.h.s of the term above to be less than $\delta_{\text{error}}$. Since $\frac{1}{N} \sigma^2 \hat{r} \le \frac{1}{L^\delta} \sigma^2$, it suffices that
\begin{align*}
\delta_{\text{error}} &\stackrel{(a)}\ge C \Big(G \delta_3^2 + \frac{1}{L^\delta} \sigma^2 \Big)\\
\implies L^\delta &\stackrel{(b)}\ge  C \Bigg( \frac{\sigma^2 }{\delta_{\text{error}} - G\delta_3^2} \Bigg) \\
\implies T &\ge C \Bigg( \frac{\sigma^2 }{\delta_{\text{error}} - G\delta_3^2} \Bigg)^{\frac{2+\delta}{\delta}} 
\end{align*}
where (a) and (b) follow for an appropriately defined $C > 0$.
\end{proof} 

\subsection{Proof of Proposition \ref{prop:lowrank_LTI_example}}
\begin{prop*}[\ref{prop:lowrank_LTI_example}] 
Let $P_{m_a}$ be a polynomial of degree $m_a$. Then,
\[
f(t)= \sum_{a=1}^A \exp{\alpha_a t} \cos(2\pi \omega_a t + \phi_a) P_{m_a}(t)
\]
admits a representation as in \eqref{eq:LTI_representation}. Further the order $\ G$ of $f(t)$ is independent of $\ T$, the number of observations, and is bounded by
\[ 
G \le A(m_{\max} + 1)(m_{\max} + 2)
\]
where $m_{\max} = \max_{a \in A} m_a$.
\end{prop*}

\begin{proof}
This proof is adapted from \cite{golyandina2001analysis}; we state it here for completeness. First, observe that if there exists latent functions $\psi_l : \{1, \dots, L\} \to \mathbb{R}$ and $\rho_l : \{1, \dots, N\} \to \mathbb{R}$ for $l \in [G]$ such that for all $(i, j) \in [L] \times [N]$
\begin{equation}\label{eq:lvm_representation}
f(i + j) = \sum_{l = 1}^G \psi_l(i) \rho_l(j), 
\end{equation}
then each $\bM^{(k)}$ (induced by $f$ for $k \in [L]$) has rank at most $G$.

Second, observe that time series that admit a representation of the form in \eqref{eq:lvm_representation} form a linear space, which is closed with respect to term-by-term multiplication, i.e.,
\begin{equation}\label{eq:term_by_term_multiply}
f(i + j) = f^{(1)} \circ f^{(2)} = \Big(\sum_{l = 1}^{G_1} \psi^{(1)}_l(i) \ \rho^{(1)}_l(j) \Big)  \Big(\sum_{l = 1}^{G_2} \psi^{(2)}_l(i) \ \rho^{(2)}_l(j) \Big),
\end{equation}
where $G_1$ and $G_2$ are the orders of the $f^{(1)}$ and $f^{(2)}$ respectively. 

Given the two observations above, it suffices to show separately that $f^{(1)}(t) = \exp{\alpha t} \cos(2\pi \omega t + \phi)$ and $f^{(2)}(t) = P_{m}(t)$ have a representation of the form in \eqref{eq:lvm_representation}.

We begin with $f^{(1)}(t) = \exp{\alpha t} \cos(2\pi \omega t + \phi)$. For $(i,j) \in [L] \times [N]$,
\begin{align*}
f^{(1)}(i + j) &= \exp{\alpha (i + j)} \cos(2\pi \omega (i+ j) + \phi) \\
&\stackrel{(a)}= \exp{\alpha i} \cos(2\pi \omega i) \cdot \exp{\alpha j} \cos(2\pi \omega j + \phi) \\
& \ \ \ - \exp{\alpha i} \sin(2\pi \omega i) \cdot \exp{\alpha j} \sin(2\pi \omega j + \phi) \\
&:= \psi_1(i) \rho_1(j) + \psi_2(i) \rho_2(j),
\end{align*} 
where in (a) we have used the trigonometric identity $\cos(a + b) = \cos(a)\cos(b) - \sin(a)\sin(b)$. Thus, for $f^{(1)}(t)$, we have $G = 2$.

For $f^{(2)}(t) = P_{m}(t)$, with $(i, j) \in [L] \times [N]$, we have $P_{m}(i + j) = \sum_{l = 0}^m c_l(i + j)^l$. By expanding $(i + j)^l$, it is easily seen (using the Binomial theorem) that there are $l + 1$ unique terms involving powers of $i$ and $j$. Hence, for $f^{(2)}(t)$, $G \le \sum_{l =1}^{m+ 1} l = \frac{(m + 1)(m + 2)}{2}$ 
\footnote{To build intuition, consider $f(t) = t^2$, in which case $f(i + j) = i^2 + j^2 + (2i)(j) := \psi_1(i) \rho_1(j) + \psi_2(i) \rho_2(j) + \psi_3(i) \rho_3(j)$. Here, $G = 3$.}. 

Now we bound $G$ for $f(t) = \sum_{a=1}^A \exp{\alpha_a t} \cos(2\pi \omega_a t + \phi_a) P_{m_a}(t)$. 
For $f^{(1)}(t) = \exp{\alpha t} \cos(2\pi \omega t + \phi) $, we have $G^{(1)} = 2$. For $f^{(2)}(t) = P_{m_a}(t)$, we have $G^{(2)} \le  \frac{(m_a + 1)(m_a + 2)}{2} \le \frac{(m_{\max} + 1)(m_{\max} + 2)}{2}$. By \eqref{eq:term_by_term_multiply}, it is clear that the order, $G^{(1, 2)}$, for $f^{(1)} \circ f^{(2)}$ is bounded by $G^{(1)} \cdot G^{(2)} \le (m_{\max} + 1)(m_{\max} + 2)$. Since there are $A$ such terms, it follows immediately that for $f(t)$, we have $G \le A (m_{\max} + 1)(m_{\max} + 2)$, which completes the proof.
\end{proof}

\subsection{Proof of Proposition \ref{prop:lowrank_harmonics}}

\begin{prop*}[\ref{prop:lowrank_harmonics}] 
For any $\epsilon \in  (0, 1)$,
\begin{itemize}
	\item[(i)] Under {\em Model Type 1}, $f^{\emph{Compact}}$ satisfies Property \ref{prop.one} with $\delta_1 = \frac{ C \mathcal{L}}{L^{\epsilon}}$ and 
	$r = L^{G\epsilon}$ for some $C > 0$.
	
	\item[(ii)] Under {\em Model Type 2}, $f^{\emph{Compact}}$ satisfies Property \ref{prop.two} with $\delta_2 = 2 \delta_1 \sqrt{N}$ and $C_\beta = 1$.
\end{itemize}
\end{prop*}

\begin{proof}
Recall $f^{\emph{Compact}} = g(\varphi(t))$ where $\varphi: \mathbb{Z} \to [-C_1, C_1]$ takes the form $\varphi(t + s) = \sum_{l=1}^{G} \alpha_l a_l(t) b_l(s)$ 
with $\alpha_l \in [-C_2, C_2], a_l: \mathbb{Z} \to [0,1], b_l: \mathbb{Z} \to [0,1]$ for some $C_1, C_2 > 0$; and $g: [-C_1, C_1] \to \mathbb{R}$ is $\mathcal{L}$-Lipschitz. Without loss of generality, we drop the dependence of $k$ on $\eta_k$ to decrease notational overload. Recall that $\eta$ (as defined in \eqref{eqref:eta_definition}) has row and column parameters $\{\theta_{1} \cdots \theta_{L}\}$ and $\{\rho_{1} \cdots \rho_{N}\}$, which denote shifts in an integer time index. 

For some $\delta > 0$, we define the set $P(\frac{\delta}{C_2 \mathcal{L}}) \subset [0, 1]^G$ such that for all $i \in [0, 1]^G$, there exists an $i' \in P(\frac{\delta}{C_2 \mathcal{L}})$ where $\norm{i - i'}_1 \le \frac{\delta}{C_2 \mathcal{L}}$. It is easily shown that we can construct this set such that $\abs{P(\frac{\delta}{C_2 \mathcal{L}})} \le (\frac{3 C_2 \mathcal{L}}{\delta})^G$. 

For any $i \in [L]$, let $\bar{a}(i) = [a_1(i), \dots, a_G(i)]$. Thus, from the construction of $P(\frac{\delta}{C_2 \mathcal{L}})$, there must exist an $\bar{a}^*(i) = [a^*_1(i), \dots, a^*_G(i)] \in P(\frac{\delta}{C_2 \mathcal{L}})$ such that $\norm{\bar{a} - \bar{a}^*}_1 \le \frac{\delta}{C_2 \mathcal{L}}$. Therefore, for any $(i, j) \in [L] \times [N]$, we have 
\begin{align} \label{eq:delta}
	\abs{\eta (i, (j-1)L ) - g \Big(\sum_{l=1}^G  \alpha_l a^*_l(i) b_l ((j-1)L) \Big)} &= \abs{f (i + (j-1)L ) - g \Big(\sum_{l=1}^G  \alpha_l a^*_l(i)  b_l( (j-1)L) \Big)} \nonumber
	\\ &= \abs{g \Big(\sum_{l=1}^G  \alpha_l a_l(i) b_l ((j-1)L) \Big) - g \Big(\sum_{l=1}^G  \alpha_l a^*_l(i) b_l ((j-1)L) \Big)} \nonumber
	\\ &\le \mathcal{L} \abs{ \sum_{l=1}^G  \alpha_l a_l(i)b_l ((j-1)L) - \sum_{l=1}^G  \alpha_l a^*_l(i)b_l ((j-1)L)} \nonumber
	\\ &= \mathcal{L} \abs{ \sum_{l=1}^G  \alpha_l \, (a_l(i) - a^*_l(i)) \cdot b_l ((j-1)L )}  \nonumber
	\\ &\le \mathcal{L}  \sum_{l=1}^G  \abs{\alpha_l \, (a_l(i) - a^*_l(i)) \cdot b_l ((j-1)L )}  \nonumber
	\\ &\le C_2 \mathcal{L}  \sum_{l=1}^G \abs{ a_l(i) - a^*_l(i)}  \nonumber
	\\ &= C_2 \mathcal{L} \,  \norm{\bar{a}(i) - \bar{a}^*(i)}_1  \nonumber
	\\ &\le \delta.  \nonumber
\end{align}

For each $(i, j) \in [L] \times [N]$, we define $\eta^*(i, (j-1)L) = g\Big(\sum_{l=1}^G  \alpha_l a^*_l(i)b_l ((j-1)L) \Big)$. Let $\bM_{(r)}$ be the matrix whose $(i, j)$-th element is $\eta^*(i, (j-1)L)$. Consequently, we have for all $k$
\begin{align*}
	\norm{\bM^{(k)}- \bM_{(r)}}_{\max} \leq \delta.
\end{align*}

Observe that for $i_1, i_2 \in [L]$, if $\bar{a}(i_1)$ and $\bar{a}(i_2)$ map to the same element $\bar{a}^*(i) \in P(\frac{\delta}{C_2 \mathcal{L}})$, then rows $i_1, i_2$ in $\bM_{(r)}$ will be identical. Therefore, there are at most $\abs{P(\frac{\delta}{C_2 \mathcal{L}})}$ distinct rows in $\bM_{(r)}$. For an appropriately defined $C > 0$, choosing $\delta =  C \mathcal{L} L^{-\epsilon}$ gives $\abs{P(\frac{\delta}{C_2 \mathcal{L}})} \le L^{G \epsilon}$. 

Hence, Property \ref{prop.three} is satisfied with $\delta_4 =  C \mathcal{L} L^{-\epsilon}$ and $r_4 = L^{G \epsilon}$. By Corollary $\ref{corollary:property3_property1and2}$, we have: 
under Model Type 1, Property \ref{prop.one} is satisfied with $\delta_1 = \delta_4$ and $r = r_4$; under Model Type 2, Property \ref{prop.two} is satisfied with $\delta_2 = 2 \delta_1 \sqrt{N}$. This completes the proof.
\end{proof}

\begin{cor*} [\ref{corollary:lowrank_harmonics_imputation}]
Under {\em Model Type 1}, let the conditions of Theorem \ref{thm:imputation} hold. Let $N = L^{1 + \delta}$ for any $\delta > 0$. Then for some $\ C > 0$ and any $\epsilon \in (0, 1)$ if 
\[ 
T \ge C \Bigg( \Big(\dfrac{1}{\delta_{\text{error}}}\Big)^{\frac{2}{1-G\epsilon}} + \Big(\frac{ \mathcal{L}}{\delta_{\text{error}}}\Big)^{\frac{1}{\epsilon}} \Bigg)^{2 + \delta}
\]
we have $\emph{MSE}(\hat{f}_I, f^{\emph{LRF}}) \le \delta_{\text{error}}$.
\end{cor*}

\begin{proof}
By Proposition \ref{prop:lowrank_harmonics}, for any $\epsilon \in (0, 1)$ and some $C_1, C_2, C_3, c_4 > 0$, 
\[
\emph{MSE}(\hat{f}_I, f^{\emph{Compact}})  \le \frac{ C_1 \sigma}{p}  \Bigg(\frac{\mathcal{L}}{L^{\epsilon}} + \frac{1}{L^{(1-G\epsilon)/2}} \Bigg) + C_2 \dfrac{(1-p)}{LNp} + C_{3} e^{-c_{4} N}. 
\]
We require the r.h.s of the term above to be less than $\delta_{\text{error}}$. Thus, we have
\begin{align*}
& \frac{ C_1 \sigma}{p}  \Bigg(\frac{\mathcal{L}}{L^{\epsilon}} + \frac{1}{L^{(1-G\epsilon)/2}} \Bigg) + C_2 \dfrac{(1-p)}{LNp} + C_{3} e^{-c_{4} N} \\
&\stackrel{(a)}\le C \Bigg( \frac{\mathcal{L}}{L^{\epsilon}}  + \frac{1}{L^{(1-G\epsilon)/2}} + \dfrac{1}{LNp} + e^{-c_{4} N}  \Bigg) \\
&\stackrel{(b)} \le C \Bigg( \frac{\mathcal{L}}{{L^{\epsilon}}} + \frac{1}{L^{(1-G\epsilon)/2}}  \Bigg)
\end{align*}
where (a) follows for an appropriately defined $C > 0$ and by absorbing $p, \sigma$ into the constant; (b) follows since $\frac{1}{LN} \le \frac{\mathcal{L}}{{L^{\epsilon}}}$, $e^{-c_{4} N} \le \frac{\mathcal{L}}{{L^{\epsilon}}}$ for sufficiently large $L, N$ and by redefining $C$. 

To have $ \frac{C}{L^{(1-G\epsilon)/2}} \le \delta_{\text{error}}/2$, it suffices that $L \ge  \Big(\frac{2C}{\delta_{\text{error}}}\Big)^{2/(1 - G\epsilon)}$. Similarly, we solve $\frac{C \mathcal{L}}{L^{\epsilon}} \le \delta_{\text{error}}/2 $ to get $L \ge \Big(\frac{2 C  \mathcal{L}}{\delta_{\text{error}}}\Big)^{\frac{1}{\epsilon}}$. Thus for appropriately defined $C$, we require $L$ to be
\begin{align}
L \ge C \Bigg( \Big(\dfrac{1}{\delta_{\text{error}}}\Big)^{\frac{2}{1 - G\epsilon}} + \Big(\frac{\mathcal{L}}{\delta_{\text{error}}}\Big)^{\frac{1}{\epsilon}} \Bigg) \\
\implies T \ge C \Bigg( \Big(\dfrac{1}{\delta_{\text{error}}}\Big)^{\frac{2}{1-G\epsilon}} + \Big(\frac{ \mathcal{L}}{\delta_{\text{error}}}\Big)^{\frac{1}{\epsilon}} \Bigg)^{2 + \delta}.
\end{align}
\end{proof}

%
%

\begin{cor*} [\ref{corollary:lowrank_harmonics_forecasting}]
Under {\em Model Type 2}, let the conditions of Theorem \ref{thm:asymptotics} hold. Let $N = L^{1 + \delta}$ for any $\delta > 0$. Then for some $\ C > 0$ and any $\epsilon \in (0, 1)$ if 
\[ 
T \ge C \Bigg( \frac{\sigma^2}{\delta_{\text{error}} - \Big(\frac{ \mathcal{L}}{L^{\epsilon}} + \delta_3\Big)^{2}} \Bigg)^{\frac{2+\delta}{\delta}}
\]
we have $\emph{MSE}(\hat{f}_F, f^{\emph{LRF}})  \le \delta_{\text{error}}$.
\end{cor*}

\begin{proof}
By Proposition \ref{prop:lowrank_harmonics}, for any $\epsilon \in (0, 1)$ and some $C > 0$,
\[
\text{MSE}(\hat{f}_F, f^{\emph{Compact}}) \le \frac{1}{N-1} \Bigg(  \Big( \frac{C \mathcal{L} }{L^{\epsilon}} + \delta_3 \Big)^2 N  + 2 \sigma^2 \hat{r} \Bigg).
\]
We require the r.h.s of the term above to be less than $\delta_{\text{error}}$. Since $\frac{1}{N} \sigma^2 \hat{r} \le \frac{1}{L^\delta} \sigma^2$, it suffices that
\begin{align*}
\delta_{\text{error}} &\stackrel{(a)}\ge C \bigg(\Big(\frac{\mathcal{L}}{L^{\epsilon}} + \delta_3\Big)^{2} + \frac{1}{L^\delta} \sigma^2 \bigg) \\
\implies L^\delta &\stackrel{(b)}\ge  C \frac{\sigma^2 }{\delta_{\text{error}} - \Big(\frac{\mathcal{L}}{L^{\epsilon}} + \delta_3\Big)^{2}} \\
\implies T &\ge C \Bigg( \frac{\sigma^2 }{\delta_{\text{error}} - \Big(\frac{\mathcal{L}}{L^{\epsilon}} + \delta_3\Big)^{2}} \Bigg)^{\frac{2+\delta}{\delta}} \\
\end{align*}
where (a) and (b) follow for an appropriately defined $C > 0$.

\end{proof}

\begin{prop*}[\ref{prop:lowrank_harmonics_example}]
\[
f^{\emph{Harmonic}}(t)= \sum_{r=1}^R \varphi_r (\sin(2 \pi \omega_r t + \phi))
\]
where $\varphi_r$ is $\mathcal{L}_r$-Lipschitz and $\omega_r$ is rational, admits a representation as in \eqref{eq:harmonics_representation}. Let $x_{\lcm}$ denote the fundamental period. Then the Lipschitz constant $\mathcal{L}$ of $f^{\emph{Harmonic}}(t)$ is bounded by 
\[
\mathcal{L} \le 2\pi \cdot \max_{r \in R}(\mathcal{L}_r) \cdot \max_{r \in R}(\omega_r) \cdot x_{\lcm}.
\] 
\end{prop*}

\begin{proof}
The fact that $f^{\emph{Harmonic}}$ has a representation as in \eqref{eq:harmonics_representation} follows immediately. It remains to show the explicit dependence of $\mathcal{L}$ on the parameters of $f^{\emph{Harmonic}}$.
Observe that
\[
f^{\emph{Harmonic}}(t) = f^{\emph{Harmonic}}(\psi(t)),
\] 
where $\psi(t) = t \mod x_{\lcm}$. By bounding the derivative of $f^{\emph{Harmonic}}(t)$, it is easy to see that
\[
\mathcal{L} \le 2\pi \cdot \max_{r \in R}(\mathcal{L}_r) \cdot \max_{r \in R}(\omega_r) \cdot x_{\lcm}.
\]
This completes the proof. 
\end{proof}

\subsection{Proof of Proposition \ref{prop:lowrank_sublinear}}

\begin{prop*}[\ref{prop:lowrank_sublinear}]
Let $\abs{\frac{d f^{\emph{Trend}}(t)}{dt}} \le C_*t^{-\alpha}$ for some $\alpha, C_* > 0$. 
Then for any $\epsilon \in (0, \alpha)$, 
\begin{itemize}
	\item[(i)] Under {\em Model Type 1}, $f^{\emph{Trend}}$ satisfies Property \ref{prop.one} with $\delta_1 = \frac{C_*}{L^{\epsilon/2}}$ and $r = L^{\epsilon/\alpha} + \frac{L - L^{\epsilon/\alpha}}{L^{\epsilon/2}}$	
	\item[(ii)] Under {\em Model Type 2}, $f^{\emph{Trend}}$ satisfies Property \ref{prop.two} with $\delta_2 = 2 \delta_1 \sqrt{N}$ and $C_\beta = 1$.
\end{itemize}
\end{prop*} 

\begin{proof}
Without loss of generality, we drop the dependence of $k$ on $\eta_k$ to decrease notational overload. Let $f(t) = f^{\text{Trend}}$. We construct our mapping $p: [L] \to [L]$ in two steps: 

{\it Step 1}: For $i < L^{\epsilon/\alpha}$, with $\epsilon \in (0, \alpha)$, let $p(i) = i$ (i.e., the $i$-th row of $\bM_{(r)}$ is equal to the $i$-th row of $\bM^{(k)}$).

{\it Step 2}: For rows $i \ge L^{\epsilon/\alpha}$, we construct the following mapping (similar to \cite{Chatterjee15}). Let $R$ and $D$ refer to the 
set of row and column parameters of the sub-matrix of $\bM^{(k)}$ corresponding to its last $L - i + 1$ rows,  $\{\theta_{L^{\epsilon/\alpha}}, \cdots, \theta_{L}\}$ 
and $\{\rho_{1,} \cdots, \rho_{N}\}$, respectively. 

Let $f'$ denote the derivative of $f$, and $\theta \in ( \min(i, i') + (j-1)L, \max(i, i') + (j-1)L )$. Then, we have that for all $i, i' \in R$
\begin{align*}
	\abs{\eta (i, (j-1)L ) - \eta (i', (j-1)L )} &= \abs{f (i + (j-1)L ) - f (i' + (j-1)L )}
	\\ &\stackrel{(a)} \le \abs{f'(\theta)} \cdot \abs{i + (j-1)L - (i' + (j-1)L)}
	\\ &\stackrel{(b)} \le C_*(L^{\epsilon/\alpha})^{-\alpha} \cdot \abs{i - i'} 
	\\ &= C_* L^{-\epsilon} \cdot \abs{i - i'},
\end{align*}
where (a) follows from the Mean Value Theorem, and (b) uses the fact that $\abs{f'(\theta)} \leq  C_* \min(i, i')^{-\alpha} \leq C_*(L^{\epsilon/\alpha})^{-\alpha}$.


We define a partition $P(\epsilon)$ of $R$ into continuous intervals of length $L^{\epsilon/2}$. Then, for any $A \in P(\epsilon)$, we have $|\theta - \theta^{'}| \le L^{\epsilon/2}$ (recall that $\theta_{i} = i$) whenever $\theta, \theta^{'} \in A$. It follows that $|P(\epsilon)| = (L -  L^{\epsilon/\alpha})/ L^{\epsilon/2} = L^{1-\epsilon/2} - L^{\epsilon(\frac{1}{\alpha} - \frac{1}{2})}$.

Let $T$ be a subset of $R$ that is constructed by selecting exactly one element from each partition in $P(\epsilon)$, i.e., $ |T| = |P(\epsilon)|$. For each $\theta \in R$, let $p(\theta)$ be the corresponding element from the same partition in $T$. Therefore, it follows that for each $\theta \in R$, we can find $p(\theta) \in T$ so that $\theta$ and $p(\theta)$ belong to the same partition of $P(\epsilon)$. 

Hence, we can define the $(i,j)$-th element of $\bM_{(r)}$ in the following way: (1) for all $i < L^{\epsilon/\alpha}$, let $p(\theta_i) = \theta_i$ such that $M^{(r)}_{ij} = \eta(\theta_i, \rho_j)$; 
(2) for $i \ge L^{\epsilon/\alpha}$, let $\bM^{(r)}_{ij} = \eta(p(\theta_i), \rho_j)$. Consequently for all $k$, 
\begin{align*}
\norm{\bM^{(k)} - \bM_{(r)}}_{\max} &\leq \max_{i \in [L], j\in[N]} \, |\eta(\theta_i, \rho_j) - \eta(p(\theta_i), \rho_j)| \\
&= \max_{i \in [j \ge L^{\epsilon/\alpha}], j\in [N]} \, |\eta(\theta_i, \rho_j) - \eta(p(\theta_i), \rho_j)| \\
&\leq \max_{i \in [j \ge L^{\epsilon/\alpha}]} \, |\theta_i - p(\theta_i)| \, L^{-\epsilon}C_* \\
&\leq C_* L^{-\epsilon/2}.
\end{align*}
Now, if $\theta_i$ and $\theta_j$ belong to the same element of $P(\epsilon)$, then $p(\theta_i)$ and $p(\theta_j)$ are identical. Therefore, there are at most $|P(\epsilon)|$ distinct rows in the last $L - L^{\epsilon/\alpha}$ rows of $\bM_{(r)}$ where $|P(\epsilon)| = L^{1-\epsilon/2} - L^{\epsilon(\frac{1}{\alpha} - \frac{1}{2})}$. Let
$\mathcal{P}(\theta) := \{p(\theta_i): i \in [L] \} \subset \{\theta_1, \dots, \theta_L \}$. By construction, since $\epsilon \in (0, \alpha)$, we have that $\abs{\mathcal{P}(\theta)} = L^{\epsilon/\alpha} + \abs{P(\epsilon)} = o(L)$. 

Hence, Property \ref{prop.three} is satisfied with $\delta_1 = \frac{C_*}{L^{\epsilon/2}}$ and $r = L^{\epsilon/\alpha} + \frac{L - L^{\epsilon/\alpha}}{L^{\epsilon/2}}$. By Corollary $\ref{corollary:property3_property1and2}$, we have: 
under Model Type 1, Property \ref{prop.one} is satisfied with $\delta_1 = \delta_4$ and $r = r_4$; under Model Type 2, Property \ref{prop.two} is satisfied with $\delta_2 = 2 \delta_1 \sqrt{N}$. This completes the proof.
\end{proof} 

\begin{cor*} [\ref{corollary:lowrank_sublinear_imputation}]
Under {\em Model Type 1}, let the conditions of Theorem \ref{thm:imputation} hold. Let $N = L^{1 + \delta}$ for any $\delta > 0$. Then for some $\ C > 0$, if 
\[ 
T \ge C \bigg(\frac{1}{\delta_{\text{error}}^{(2(\alpha + 1)/\alpha)}}\bigg)^{2 + \delta}
\]
we have $\emph{MSE}(\hat{f}_I, f^{\emph{LRF}}) \le \delta_{\text{error}}$.
\end{cor*}

\begin{proof}
By Proposition \ref{prop:lowrank_sublinear}, for any $\epsilon \in (0, \alpha)$ and some $C_1, C_2, C_3, c_4 > 0$,
\begin{align*}
	\emph{MSE}(\hat{f}_I, f^{\emph{Trend}}) &\le \frac{C_1 \sigma}{p} \Bigg( \frac{C_*}{L^{\epsilon/2}} + \frac{1}{(L^{1 - \epsilon/\alpha} + L^{\epsilon/2} )^{1/2}} \Bigg) \\
	&+ C_2 \dfrac{(1-p)}{LNp} + C_{3} e^{-c_{4} N}.
\end{align*}
We require the r.h.s of the term above to be less than $\delta_{\text{error}}$. We have,
\begin{align*}
&  \frac{C_1 \sigma}{p} \Bigg( \frac{C_*}{\sqrt{p}L^{\epsilon/2}} + \frac{1}{\sqrt{p} (L^{1 - \epsilon/\alpha} + L^{\epsilon/2} )^{1/2}} \Bigg) + C_2 \dfrac{(1-p)}{LNp} + C_{3} e^{-c_{4} N}  \\
&\stackrel{(a)}\le C \Bigg( \frac{1}{L^{\epsilon/2}} + \frac{1}{ (L^{1 - \epsilon/\alpha} + L^{\epsilon/2} )^{1/2}} + \dfrac{1}{LN} + e^{-c_{4} N} \Bigg) \\
&\stackrel{(b)} \le C \Bigg( \frac{1}{L^{\epsilon/2}} + \frac{1}{ (L^{1 - \epsilon/\alpha} + L^{\epsilon/2} )^{1/2}}  \Bigg) \\
&\le C \Bigg( \frac{1}{L^{\epsilon/2}} + \frac{1}{ (L^{1 - \epsilon/\alpha} )^{1/2}}  \Bigg)
\end{align*}
where (a) follows for an appropriately defined $C > 0$ and by absorbing $p, \sigma$ into the constant; (b) follows since $\frac{1}{LN} \le \frac{1}{L^{\epsilon/2}}$, $e^{-c_{4} N} \le \frac{1}{L^{\epsilon/2}}$ for sufficiently large $L, N$ and by redefining $C$. 

Setting $\frac{\epsilon}{2} = \frac{1 - \epsilon/\alpha}{2}$, we get $\epsilon = \frac{\alpha}{\alpha + 1} < \alpha$, which satisfies the condition that $\epsilon \in (0, \alpha)$ in Proposition \ref{prop:lowrank_sublinear}. Therefore, it suffices that $\delta_{\text{error}} \ge C L^{\frac{\alpha}{2(\alpha + 1)}} \implies \ T \ge C \Big(\frac{1}{\delta_{\text{error}}^{\frac{2(\alpha + 1)}{\alpha}}}\Big)^{2 + \delta}$.

\end{proof}

%
%

\begin{cor*} [\ref{corollary:lowrank_sublinear_forecasting}]
Under {\em Model Type 2}, let the conditions of Theorem \ref{thm:asymptotics} hold.. Let $N = L^{1 + \delta}$ for any $\delta > 0$. Then for some $\ C > 0$ and for any $\epsilon \in (0, \alpha)$ if 
\[ 
T \ge C \Bigg( \frac{\sigma^2}{\delta_{\text{error}} - \Big(\frac{1}{L^{\epsilon/2}} + \delta_3\Big)^{2}} \Bigg)^{\frac{2+\delta}{\delta}}
\]
we have $\emph{MSE}(\hat{f}_F, f^{\emph{LRF}})  \le \delta_{\text{error}}$.
\end{cor*}

\begin{proof}
By Proposition \ref{prop:lowrank_sublinear}, for any $\epsilon \in (0, \alpha)$,
\[
\emph{MSE}(\hat{f}_F, f^{\emph{Trend}}) \le \frac{1}{N-1} \Big( ( \frac{ C_*}{L^{\epsilon/2}} + \delta_3 )^2 N + 2 \sigma^2 \hat{r} \Big).
\]
We require the r.h.s of the term above to be less than $\delta_{\text{error}}$. Since $\frac{1}{N} \sigma^2 \hat{r} \le \frac{1}{L^\delta} \sigma^2$, it suffices that
\begin{align*}
\delta_{\text{error}} &\stackrel{(a)}\ge C \Bigg(\Big(\frac{1}{L^{\epsilon/2}} + \delta_3\Big)^{2} + \frac{1}{L^\delta} \sigma^2 \Bigg) \\
\implies L^\delta &\stackrel{(b)}\ge  C \frac{\sigma^2 }{\delta_{\text{error}} - \Big(\frac{1}{L^{\epsilon/2}} + \delta_3\Big)^{2}} \\
\implies T &\ge C \Bigg( \frac{\sigma^2 }{\delta_{\text{error}} - \Big(\frac{1}{L^{\epsilon/2}} + \delta_3\Big)^{2}} \Bigg)^{\frac{2+\delta}{\delta}} 
\end{align*}
where (a) and (b) follow for an appropriately defined $C > 0$.
\end{proof}


\begin{prop*}[\ref{prop:lowrank_sublinear_example}]
For $t \in \mathbb{Z}$ with $\alpha_b < 1$ for $b \in [B]$,
\[
f^{\text{Trend}}(t) = \sum_{b=1}^B \gamma_b t^{\alpha_b} + \sum_{q=1}^{Q}\log(\gamma_q t).
\]
admits a representation as in \eqref{eq:sublinear_representation}. 
\end{prop*}

\begin{proof}
The proof follows immediately from the definition of $f^{\text{Trend}}$.
\end{proof}

\subsection{Proof of Proposition \ref{prop:lowrank_additive_mixture}}
\begin{prop*}[\ref{prop:lowrank_additive_mixture}]
	Under {\em Model Type 1}, $f^{\emph{Mixture}}$ satisfies Property \ref{prop.one} with $\delta_1 = \sum_{q=1}^{Q} \rho_q \delta_q$  and  $r = \sum_{q=1}^{Q} r_q$.
\end{prop*} 

\begin{proof}
Let $\bM^{(1)}_{g}$ refer to the underlying mean matrix induced by each $X_g (t)$. Similarly, as defined in Property \ref{prop.one}, let $\bM_{g, (r)}$ be the low rank matrix associated with $\bM^{(1)}_{g}$. We have
\begin{align*}
	\bM^{(1)} = \sum_g^G \alpha_g \bM^{(1)}_{g}.
\end{align*}
We define $\bM_{(r)}$ as
\begin{align*}
	\bM_{(r)} = \sum_g^G \alpha_g \bM_{g, (r)}.
\end{align*}
As a result, we have that $
 \text{rank}(\bM_{(r)}) \le \sum_g^G r_g,
$
and 
\begin{align*}
\norm{\bM^{(1)} - \bM_{(r)}}_{\max} &= \norm{\sum_g^G \alpha_g \bM^{(1)}_{g} - \sum_g^G \alpha_g \bM_{g, (r)}}_{\max}  \\
&\leq \sum_g^G \alpha_g \norm{\bM^{(1)}_{g}  - \bM_{g, (r)}}_{\max} \\
&= \sum_g^G \alpha_g \delta_g.
\end{align*}
This completes the proof.
\end{proof}

\end{document}